\documentclass[11pt,reqno]{amsart}

\usepackage[margin=1in]{geometry}
\usepackage{nicefrac, graphicx, mathrsfs}
\usepackage{amsmath, amssymb, amsthm}
\usepackage[foot]{amsaddr}
\usepackage[font=small]{caption}

\usepackage{physics}
\usepackage{enumitem}
\usepackage{xcolor}
\definecolor{FTChampagnePink}{HTML}{F2DFCE}
\definecolor{FTOldLace}{HTML}{FFF1E0}
\definecolor{FTFloralWhite}{HTML}{FFF9F5}
\definecolor{FTMetallicSeaweed}{HTML}{0D7680}
\definecolor{FTMaroon}{HTML}{8F223A}
\definecolor{BBG}{HTML}{274C43}
\makeatletter
\newcommand{\globalcolor}[1]{%
  \color{#1}\global\let\default@color\current@color
}
\makeatother

\usepackage{todonotes}

\usepackage[sc]{mathpazo}\renewcommand{\mathbf}{\mathbold}
\usepackage{eucal}

\usepackage{thmtools}
\usepackage{thm-restate}

\usepackage{tikz}
\usetikzlibrary{cd}
\usetikzlibrary{matrix,arrows}

\usepackage[all,2cell,cmtip]{xy}
\UseAllTwocells

\usepackage{hyperref}
\hypersetup
{
    colorlinks,	%
    citecolor=FTMetallicSeaweed,%
    filecolor=black,%
    linkcolor=FTMaroon,%
    urlcolor=FTMaroon
}

\usepackage{cleveref}

\theoremstyle{plain}
  \newtheorem{theorem}{Theorem}[section]
  
  \newtheorem{lemma}[theorem]{Lemma}
  \newtheorem{fact}[theorem]{Fact}
  \newtheorem{proposition}[theorem]{Proposition}

\theoremstyle{definition}
  \newtheorem{definition}[theorem]{Definition}
  
  \newtheorem{ex}[theorem]{Example}
  
  \newtheorem{remark}[theorem]{Remark}
  \newenvironment{example}{\begin{ex}}{\end{ex}}

  \newcommand{\sett}[2]{\qty{#1 \ \middle | \ #2}}

\newcommand{\C}{\mathbb{C}}
  \newcommand{\R}{\mathbb{R}}

  \newcommand{\manif}{\mathfrak{M}}

	\newcommand{\inv}[1]{{{#1}^{-1}}}

	\newcommand{\X}{\mathfrak{X}}

	\newcommand{\Fc}[1]{\Fc[#1]} 

\newcommand{\innerprod}[2]{\expval{{#1}, {#2}}}

\newcommand{\fo}[1]{#1^{(1)}}

\newcommand{\Fourier}{\mathcal{F}}
\newcommand{\Power}{\mathcal{P}}

\DeclareMathOperator*{\spec}{\mathrm{spec}}

\newcommand{\openball}[2]{B(#1,#2)}

\title{Power Spectrum Signatures of Graphs}

\author{Karamatou Yacoubou Djima$^\dag$}
\author{Ka Man Yim$^\ddag$}
\address{$^\dag$Mathematics \& Statistics, Wellesley College, 106 Central St, Wellesley, MA, 02481, USA}
\address{$^\ddag$School of Mathematics, Cardiff University, Cardiff, CF24 4AG, United Kingdom}
\email{ky105@wellesley.edu,yimkm@cardiff.ac.uk}

\begin{document}

\begin{abstract}

Point signatures based on the Laplacian operators on graphs, point clouds, and manifolds have become popular tools in machine learning for graphs, clustering, and shape analysis. In this work, we propose a novel point signature, the power spectrum signature, a measure on $\mathbb{R}$ defined as the squared graph Fourier transform of a graph signal. Unlike eigenvectors of the Laplacian from which it is derived, the power spectrum signature is invariant under graph automorphisms. We show that the power spectrum signature is stable under perturbations of the input graph with respect to the Wasserstein metric. We focus on the signature applied to classes of indicator functions, and its applications to generating descriptive features for vertices of graphs. To demonstrate the practical value of our signature, we showcase several applications in characterizing geometry and symmetries in point cloud data, and graph regression problems.     
\end{abstract}

\maketitle

{\small{\it Keywords:} Power spectrum, Point signatures, Graph Laplacian, Wasserstein distance, Diffusion kernels, Manifold learning, Graph regression}


\section{Introduction}

In many data applications such as shape analysis or the study of biological or transportation networks, graphs are increasingly preferred over Euclidean spaces as they provide a simple, natural way to store information about a network of objects. Extracting the topological features of these graphs and classifying them is important, for example, in predicting the chemical properties of compounds using their molecular structure. Since the atoms that compose proteins and other molecules are defined only up to rigid motion--they have no fixed orientation--the geometry of a molecule may be more effectively represented by the pairwise distances between atoms rather than their spatial positions. Objects or data with such structures abound in applications such as shape classification and retrieval, where the shape data often come in a discrete mesh (graph) representation, and two shapes are deemed similar if rigid or isometric transformations can map one onto the other. Graph models, with the graph domain sometimes built from manifold discretization, have also been used in applications such as dimension reduction and clustering. The primary challenge in graph classification lies in the lack of the spatial and geometric intuition inherent to Euclidean spaces, making it difficult to define and compute meaningful measures of distance or similarity. Therefore, tackling graph classification requires methods that extend beyond traditional machine learning techniques designed for Euclidean data. 

Considerable research has focused on approaches involving graph kernel methods \cite{DBLP:journals/corr/abs-1903-11835,JMLR:v11:vishwanathan10a}. The crux of these approaches is the graph Laplacian, whose eigenvectors and eigenvalues encode intrinsic topological
and geometric information about graphs. Using the Laplacian eigendecomposition, these techniques embed graphs into Euclidean spaces and use these embeddings for classification and other tasks, often using neural networks. For example, graph kernels are involved in designing point signatures, i.e., vertex features such as the heat kernel signature \cite{Sun2009ADiffusion}, which are invariant under rigid or isometric transformations, and can be used for classification or shape matching. In~\cite{LiChunyuan2013Amdf}, the authors propose a graph's spectral wavelet signature that allows the analysis and design
of efficient shape signatures for nonrigid 3D shape retrieval. In~\cite{Carriere2020Perslay:Signatures,Yim_2021}, the authors introduce a framework that combines the advantages of kernel signatures and persistent homology to produce a persistence-based graph classifier. In Laplacian eigenmaps~\cite{belkinniyogiLA2023}, diffusion maps~\cite{DBLP:journals/corr/KipfW16}, and other kernel methods~\cite{ham2004kernel,filippone2008survey}), the authors design manifold learning techniques mainly for dimension reduction and clustering. A notable advantage of diffusion maps, a manifold-learning technique introduced in~\cite{Coifman2006DiffusionMaps}, is the formulation of a framework that bridges classical harmonic analysis with its counterparts on graphs and manifolds. This framework has been further developed in related works, leading to methods such as diffusion wavelets~\cite{coifmanmaggion2006} and diffusion wavelet packets \cite{BREMER200695}. 
Continuing this line of connecting the ``old'' to the ``new'', \cite{RICAUD2019474} provides a comprehensive overview of the state of the art in graph signal processing, with a particular emphasis on defining, interpreting, and applying the Fourier transform~\cite{DBLP:journals/corr/abs-1211-0053} to data on graphs. Diffusion processes and harmonic analysis on graphs can be seen as one of the foundations for graph neural networks (GNNs) and geometric deep learning. For example, architectures such as \cite{DBLP:journals/corr/KipfW16,DBLP:journals/corr/abs-1904-07785} incorporate the graph Laplacian in the message passing routines. Even more generally, the paradigm of graph neural network architectures are essentially non-linear convolutions of graph signals, analogous to how convolutional neural networks draw on harmonic analysis on Euclidean domains~\cite{Bronstein2021GeometricGauges}.   

Another framework with increasing popularity in machine learning comes from methods based on optimal transport \cite{Santambrogio2015OptimalMathematicians}, first defined as a way to compare probability distributions defined over either the same ground
space or multiple pre-registered ground spaces. In practice, these methods have enabled the construction of precise generative models for signal intensities and other data distributions, making them valuable in various applications, including statistical machine learning and image processing \cite{Kolourietal2017}. The Wasserstein distance or Kantorovich-Rubenstein metric is a transport-based distance function defined between probability distributions on some metric space. This distance has gained much traction due to its geometric characteristics that inspired new methods for matching and interpreting the meaning of data distributions.

\subsection{Contribution and Organization}

In this paper, we propose a novel set of vertex features derived from the graph Laplacian spectrum and eigenvectors. We demonstrate how these new signatures can be applied to characterize geometric features associated to points on point clouds and graphs and create informative features of vertices for machine learning problems for graphs. We demonstrate the stability of these features under perturbations of the input graph or point cloud. 

In~\Cref{sec:powerspectrumsignaturesintro}, we formally define this signature and establish its key properties. Specifically, in~\Cref{subsec:defnsProp}, we demonstrate that the power spectrum signature
evaluated on vertices is invariant under automorphisms of graphs. In \Cref{subsec:connections}, we discuss the relationship between other signatures and ours, with a particular focus on their perturbation properties.

Then, in~\Cref{sec:perturbationtheory}, we show that our approach provides a stable characterization of functions on graphs under perturbations of the graph Laplacian. Considering power spectrum signatures as elements of 1-Wasserstein space, we derive in \Cref{thm:Lipschiz} an explicit bound on the variation of the signature in terms of the size of the perturbation to the Laplacian. Our precise perturbation theory analysis result stands out because it holds with minimal assumptions, even in the presence of spectral degeneracies. %

Finally, we validate our theoretical results through computational experiments on real and synthetic data given in~\Cref{sec:applications}. These experiments demonstrate that power spectrum signatures of vertices add valuable features that improve the performance of graph neural network architectures in benchmark graph regression problems. Using point cloud data, we also demonstrate how the signature can identify points that are related by approximate global isometries of the point cloud, and distinguish points with different local geometries.


\section{Power Spectrum Signatures}\label{sec:powerspectrumsignaturesintro}

\subsection{Preliminaries and Notation}\label{subsec:PreliminariesNotation}

The power spectrum signature of a graph is constructed using its Laplacian operator. In this section, we present some basic properties of the graph Laplacian required to define this signature. Let $G = (V, E)$ be a simple, undirected finite graph with $n = |V|$ nodes. The graph {\it adjacency matrix} $A$, is a symmetric matrix that describes the connection between nodes $i$ and $j$, such that $A_{i,j} = 1$ if there is an edge between $i$ and $j$ and $A_{i,j} = 0$ otherwise. Note that the adjacency matrix can also be more generally defined as $A_{i,j} = w(i,j)$, where $w$ is a weight associated with the edge $(i,j) \in E$. Often, the weights come from a symmetric, positive semi-definite kernel function $k(i,j)$ that encodes some notion of local geometry that can be application-dependent. 
Given $A$, we can define the diagonal {\it degree matrix} $D$ with $D_{i,i} = \sum_{j} A_{i,j}$. The {\it graph Laplacian} matrix $L$ of $G(V,E)$ is then defined as $L = D - A$ and the {\it normalized Laplacian} matrix is $L = I- D^{-1/2}AD^{-1/2}$, where $I$ is the $n \times n$ identity matrix. 

Throughout this work, we will use the normalized Laplacian. Since $L$ is a real symmetric matrix, we know by the spectral theorem that it has a complete set of orthonormal eigenvectors $\phi_i$ corresponding to a set of real, non-negative eigenvalues $\lambda_i$. In particular, the eigenvalues of $L$ are confined to the interval $[0,\,2]$.

A graph or vertex function $f: V \to \R$ can be defined on $G = (V, E)$ by assigning a real value to each vertex. We use the notation $f(x)$ or $f_x$ to designate the value of $f$ at the vertex $x \in V$. When an ordering of the vertices is fixed, this function can be regarded as a vector $f \in \R^n$. We use $\delta_x: V \to \R$ to denote the indicator function on vertex $x \in V$.  

\subsection{Definitions \& Properties}\label{subsec:defnsProp}

Consider the simple, undirected finite graph $G = (V, E)$, $n = |V|$, with a normalized graph Laplacian $L$ that encodes some notion of similarity between the points on the graph as explained above. We assume that $L$ has an eigendecomposition of the form 
\[ 
    L = \sum_{i = 1}^{n} \lambda_i \phi_i^{\phantom{}} \phi_i^\intercal. 
\]
Then for the choice of eigenbasis $\qty{\phi_i}_{i = 1}^n$ in the formula above, the graph Fourier transform (GFT) $\Fourier: \R^n \to M(\R)$, defined as 
\begin{equation}\label{eq:Fourier}
    \mathcal{F} \qty[f](t) = \sum_{i = 1}^{n} \delta(t - \lambda_i) \innerprod{\phi_i}{f},
\end{equation}
sends a vertex function $f \in \R^{n}$ to a discrete, signed measure on $[0,2]$, the support of the eigenvalues. The GFT is usually expressed more plainly in terms of a single inner product, but we choose the above notation to stress the measure-theoretic perspective of our work. 

Note that the Fourier transform is dependent on a particular choice of basis. For example, if we changed the sign of each eigenvector, there would be a corresponding sign change in the Fourier transform. In addition, different choices of eigenbases in degenerate eigenspaces, i.e., eigenspaces corresponding to eigenvalues with multiplicity greater than one, would also change $\mathcal{F}$. However, what \emph{is} invariant to our choice of basis is the power spectrum of the Fourier transform. We now introduce it formally as our spectral signature.

\begin{definition}[{\emph{Power Spectrum Signature}}]
    Let $H$ be an $n \times n$ Hermitian matrix, and suppose we are given its orthonormal eigenbasis $\phi_1,\,\hdots,\,\phi_{n}$, with corresponding eigenvalues $\lambda_1,\,\,\hdots,\,\lambda_{n}$. The  {\it power spectrum} $\Power: \R^n \to \R$ maps a function $f \in \R^{n}$ with unit norm $\norm{f} = 1$ to a discrete probability measure on $\R$, and is defined as
    \begin{align}
        \label{eqn:powerspectrum}
        \Power \qty[f](t) 
        &= \vert \mathcal{F} \qty[f](t) \vert^2 \nonumber \\
        &= \sum_{i = 1}^n \delta(t - \lambda_i) \innerprod{\phi_i}{f}^2 \\
        &=:\mu^H_f, \nonumber
    \end{align} 
    with the notation $\mu^H_f$ used in probabilistic settings. In particular, if $H$ is a graph Laplacian, we say that $\Power \qty[f]$ is the {\it power spectrum signature} of the vertex function $f: V \to \R$.
\end{definition}
\begin{example}[A power spectrum computation]
    Consider the vertex function $f = \delta_x + \delta_y$ defined on a graph $G(V, E)$.  
    The power spectrum of $f$ can be calculated as
    \[
        \Power\qty[f](t) = \sum\limits_{i = 1}^{n} \delta(t - \lambda_i)[\phi_i^2(x) + 2\phi_i(x)\phi_i(y)+ \phi_i^2(y)].
    \]
\end{example}
Power spectra describe the distribution of power among the frequency components that make up a signal $f$. This information is crucial in many applications, making power spectra widely used in real-world scenarios to analyze the frequency content of signals (in our case, the vertex functions), highlight dominant frequencies, or characterize noise levels. For example, spectral analysis can identify the pitch and tone of a musical instrument. Furthermore, noise levels often correspond to higher frequencies in the power spectrum and can be reduced by filtering out those frequencies from the signal. Because they represent essential information on important dynamics, power spectra constitute a natural and compelling characterization of systems. 

\begin{remark}
        Unlike the Fourier transform, the power spectrum signature of a graph is independent of the choice of orthonormal eigenbasis for the Laplacian because, for a given eigenspace, the sum $\innerprod{\phi_i}{f}^2$ over the orthonormal eigenbasis $\{ \phi_i\}_{i = 1}^{n}$ of the eigenspace is simply the norm of the projection of $f$ onto that eigenspace, which is independent of the choice of eigenbasis. If we let $\spec(H)$ be the \emph{set} of unique eigenvalues of a Hermitian matrix $H$, and $P_\lambda \in \C^{n \times n}$ be the projection matrix onto the (potentially degenerate) eigenspace of $\lambda \in \spec(H)$, then we can rewrite $\Power[f](t)$ as  
\begin{equation}\label{eq:probmeasureunderHerm}
    \Power[f](t) = \sum\limits_{\lambda \in \spec(H)}\langle f, P_\lambda f \rangle \delta(t - \lambda).
    \end{equation}
    This expression explicitly relates our power spectrum signature with the \emph{spectral invariants} considered by~\cite{Furer2010OnInvariants}, which consider collections of projection operators onto eigenspaces along with the eigenvalues as graph invariants. In particular, we note the results in ~\cite{Rattan2023Weisfeiler-LemanSpectra} which establish the strength of spectral invariants as graph isomorphism tests within the Weisfeiler-Lehman hierarchy. 
\end{remark}

In this work, we focus on the power spectrum of indicator functions $\delta_x: V \to \R$ on $V$, and use them as features of the vertex $x$. So we will sometimes refer to the power spectrum signature $\Power[\delta_x]$ as the \emph{spectral feature} of vertex $x$. We will also use $\mu_x^H := \Power[\delta_x]$ for spectral features at vertex $x$. \Cref{fig:vertexfcnPS} shows a graph function and the associated power spectrum.

\begin{figure}[!htbp]
    \centering
    \includegraphics[width=0.5\linewidth]{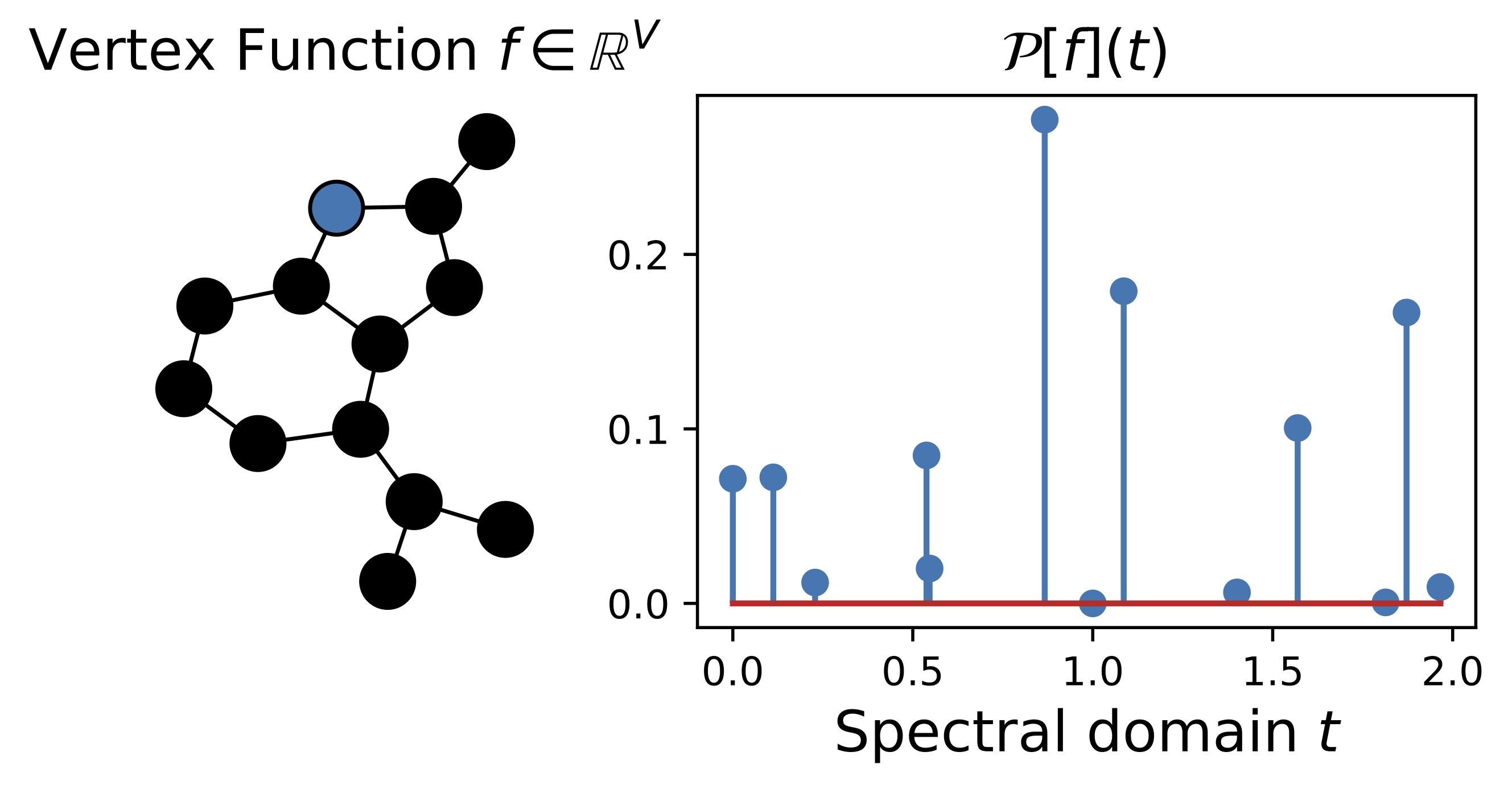}\vspace{-.5cm}
    \caption{The power spectrum of the vertex function supported on the blue vertex.}
    \label{fig:vertexfcnPS}
\end{figure}

Next, we show how a graph can be described completely by finitely many power spectra. 

\begin{proposition}[Injectivity] \label{prop:injectivity} Let $f_{x,y}$ denote 
\begin{equation} \label{eq:pair_indicators}
    f_{x,y} = 
     \begin{cases}
        \frac{1}{\sqrt{2}}\qty(\delta_{x} + \delta_{y}), & x \neq y, \\
        \delta_{x}, & x =y.
    \end{cases}
\end{equation}
Given the power spectra $\qty(\Power\qty[f_{x,y}])_{\qty{x,y} \subset V}$  with respect to a Hermitian matrix $H$, we can recover the original Hermitian matrix via the formulae 
\begin{equation}
     H_{xy} =
     \begin{cases}
     \int t \Power\qty[f_{x,y}]) \dd{t}  - \frac{1}{2} \qty(\int t \Power\qty[f_{x}]\dd{t} + \int t \Power\qty[f_{y}]\dd{t}), & x \neq y, \\
     \int t \Power\qty[f_{x}]) \dd{t}, & \text{o.w.}.         
     \end{cases}
\end{equation}
\end{proposition}
\begin{proof}
This result follows from direct computation. For a given eigendecomposition of $H$, let us express each entry of $H$ in terms of its eigenvectors and eigenvalues. For the off-diagonal entries of $H$ where $x \neq y$, we have
\begin{align*}
    H_{xy} &= \sum\limits_{i = 1}^n \lambda_i\phi_{ix} \phi_{iy}^\ast\\
            &= \sum\limits_{i = 1}^n \lambda_i \frac{\qty(\phi_{ix} + \phi_{iy})^2 - \phi_{ix}^2 - \phi_{iy}^2}{2} \\
            &= \sum\limits_{i = 1}^n \lambda_i \innerprod{f_{x,y}}{\phi_i}^2 - \frac{1}{2}\qty(\sum_i \lambda_i \innerprod{f_{x}}{\phi_i}^2 + \sum_i \lambda_i \innerprod{f_{y}}{\phi_i}^2) \\
            &= \int t \Power\qty[f_{x,y}]) \dd{t}  - \frac{1}{2} \qty(\int t \Power\qty[f_{x}]\dd{t} + \int t \Power\qty[f_{y}]\dd{t});
\end{align*}
and for the diagonal entries,
\begin{align*}
    H_{xx} &= \sum_i \lambda_i \phi_{ix}^2 =  \sum_i \lambda_i \innerprod{f_x}{\phi_i}^2=\int t \Power\qty[f_{x}]) \dd{t}.
\end{align*}
\end{proof}

The next desirable attribute that we verify for our signature is permutation invariance. Let $\Phi_\sigma$ be the $n \times n$ column permutation matrix representing $\sigma \in S_n$, where $S_n$ is the symmetric group on $n$. Recall that $S_n$ acts on $\C^n$ by permuting the coordinates of a function $f \in \C^n$ via the action $(\Phi_\sigma^T  f)_i = f_{\sigma(i)}$; similarly, its action on $n \times n$ matrices is given by $M \mapsto \Phi_\sigma^T M \Phi_\sigma$. In terms of entries, $ (\Phi_\sigma^T M \Phi_\sigma)_{ij} = M_{\sigma(i) \sigma(j)}$. The permutation automorphism group $\mathrm{Aut}(M)$ of $M$ is the stabilizer subgroup of $M$; that is, the group of permutations $\sigma$ such that $\Phi_\sigma^T M \Phi_\sigma  = M$. In terms of entries, $M_{\sigma(i) \sigma(j)} = M_{ij}$. In particular, if $M$ is the adjacency matrix or Laplacian of a graph $G$, then, by definition, $\mathrm{Aut}(M) = \mathrm{Aut}(G)$, where $\mathrm{Aut}(G)$ denotes the automorphism group of $G$.


\begin{proposition}[Automorphism invariance] Let $H \in \C^{n \times n}$ be a Hermitian matrix. Consider the action of $\mathrm{Aut}(H)$ on vectors $f \in \C^n$ with unit norm. If $g$ is in the orbit of $f$ under the action of $\mathrm{Aut}(H)$, then $\Power[f] = \Power[g]$. Conversely, for $f_{x,y}$ indexed over $x,y \in 1, \ldots ,n$ as defined in \cref{eq:pair_indicators}, if $\sigma$ is a permutation satisfying $\Power[f_{x,y}] = \Power[f_{\sigma(x),\sigma(y)}]$, for all pairs $(x,y)$, then $\sigma \in \mathrm{Aut}(H)$. 
\end{proposition}

\begin{proof} This is a consequence of simple facts of linear algebra. If $\Phi$ is unitary, i.e., $\Phi^\ast= \inv{\Phi}$, then $H v = \lambda v$ implies $\Phi^\ast v$ is an eigenvector of $\Phi^\ast H \Phi$ with eigenvalue $\lambda$. Suppose $v_1, \ldots, v_m$ is an orthonormal basis for the $\lambda$-eigenspace of $H$. If $\Phi^\ast H \Phi = H$, then $\Phi^\ast v_1, \ldots,\Phi^\ast v_m$ is also an orthonormal basis for the $\lambda$-eigenspace of $H$. We note that the projection on the $\lambda$-eigenspace can be equivalently written as $P_\lambda = \sum vv^\ast$ and $P_\lambda = \sum (\Phi^\ast v)(\Phi^\ast v)^\ast $. Thus, the invariance  $\Phi^\ast H \Phi = H$ translates into an invariance on projections $P_\lambda = \sum (\Phi^\ast v)(\Phi^\ast v)^\ast = \Phi^\ast  \qty(\sum v v^\ast)  \Phi  = \Phi^\ast P_\lambda \Phi$. Thus, for $f \in \C^n$ and $\Phi$ unitary, $(\Phi f)^\ast P_\lambda (\Phi f) = f^\ast (\Phi^\ast P_\lambda \Phi) f = f^\ast P_\lambda f$. Writing $\Power[f](t) = \sum_{\lambda \in \spec (H)} \langle f, P_\lambda f \rangle \delta(t - \lambda)$, it follows from $(\Phi f)^\ast P_\lambda (\Phi f)  = f^\ast P_\lambda f$ that $\Power[f] = \Power[\Phi f]$.

For the converse, if $\Power[f_{x,y}] = \Power[f_{\sigma(x),\sigma(y)}]$ for all pairs $x,y \in  1, \ldots, n$, then by \Cref{prop:injectivity}, we have $M_{xy} = M_{\sigma(x)\sigma(y)}$; therefore $\sigma \in \mathrm{Aut}(H)$. 
\end{proof}

\subsection{Connections to other point signatures}\label{subsec:connections}

Point signatures are typically designed to summarize the shape of a local neighborhood of a vertex or point under investigation. Notable examples include spin images, multi-scale local surface signatures, and integral invariant signatures. A comprehensive survey of these methods can be found in~\cite{Sun2009ADiffusion}. Here, we draw connections between the power spectrum signature at a vertex and other vertex characterizations such as the heat kernel signature~\cite{Sun2009ADiffusion}, diffusion maps~\cite{Coifman2006DiffusionMaps}, the global point signature~\cite{Rustamov07}, and the wavelet signature~\cite{LiChunyuan2013Amdf}. 


Before describing individual methods and contrasting them, it is helpful to highlight the common threads emphasized to varying degrees in their constructions: kernels, related operators, and their eigendecomposition. In other words, despite differences in technique, these methods manipulate the same intrinsic information derived from the eigendecomposition of a linear operator.

Assume that we have a fixed domain $X$ (graph or manifold) equipped with a kernel $k: X \times X \to \R$ with two main properties: 1) symmetry, i.e., for $x,\, y \in X$, $k(x,y) = k(y,x)$, and 2) positive-definiteness, i.e., for any finite subset  $\{x_1, \ldots, x_n \}$ the Gram matrix $K_{ij} := k(x_i,x_j)$ is positive-semi-definite. A common example is the heat kernel associated with the Laplace-Beltrami operator $\Delta$ of a compact Riemannian manifold $\manif$. Consider the heat equation
\begin{equation}\label{eqn:heatequation}
    \begin{aligned}
        \Delta u(x,t) &= -\frac{\partial u(x,t)}{\partial t}, \\
        u(x,0)&=f(x), 
    \end{aligned}
\end{equation}
 The solution of this equation is given by the convolution 
\begin{align}\label{eqn:solnofheateqn}
    u(x,t) &= e^{-t\Delta}f(x) = \int k_t(x,y) f(y) dy.
\end{align}
The function $k_t(x,y)$ in~\cref{eqn:solnofheateqn} is the heat kernel, obtained by solving the PDE~\eqref{eqn:heatequation} where $f$ is the Dirac measure $\delta_x$ on $x$. 
When $\manif$ is compact, the heat kernel can be written using the spectral decomposition of $\Delta$ as
\begin{equation}\label{eq:heatkernelform}
    k_t(x,y) = \sum\limits_{i = 0}^\infty e^{-\lambda_i t} \phi_i(x)\phi_i(y),
\end{equation}   
where $\lambda_i$ and $\phi_i$ are the eigenvalues and eigenfunctions of $\Delta$, respectively. \Cref{eq:heatkernelform} illustrates the interplay between kernels and their corresponding self-adjoint, compact linear integral operators through spectral decomposition. 
The spectral decomposition of the Laplace-Beltrami operator on the real line also famously appears in the Fourier transform of $L^2$-functions. For simplicity, assume $\manif = \R$, so $\Delta = \frac{d^2}{dx^2}$, and we have
\[
    \frac{d^2}{dx^2} e^{2\pi i  x \xi} = -(2\pi \xi)^2 e^{2\pi i x \xi},
\]
where $e^{2\pi i x \xi}$ is an eigenfunction of $\Delta$ with corresponding eigenvalue $-(2\pi \xi)^2$. For a function $f \in L^2(\R)$, the Fourier transform is defined as
\begin{align*}
    \mathcal{F}[f](\xi) 
    &= \int_\R f(x) e^{-2\pi i x \xi}~dx= \langle f, e^{2\pi i x \xi} \rangle,
\end{align*}
where $\langle \cdot, \cdot \rangle$ is the $L^2$ inner product. 

The graph Fourier transform is defined by analogy to the classical Fourier transform. First, observe that the graph Laplacian of a graph $G = (V,E)$ can be viewed as a discretization of the Laplace-Beltrami operator, acting on a vertex function $f \in \R^V$ by
\[
    Lf(x) = \sum\limits_{x \sim y} w(x,y) [f(x) - f(y)],
\]
where $x \sim y$ means that $x$ is adjacent to $y$, and $w(x,y)$ represent edge weights. From there, we use the eigendecomposition of $L$ described earlier to define the graph Fourier transform as 
\begin{align*}
    \mathcal{F}[f](\lambda_i) 
    &= \langle f, \phi_i \rangle,
\end{align*}
which is equivalent to \cref{eq:Fourier}. 

We now describe the point signatures and show how their formulas are related. As stability is one of the main focus of our work, we also discuss this question. However, our power spectra signature stability will be discussed in greater detail in \Cref{sec:perturbationtheory}. In particular, in contrast to some signatures discussed in the following, our signature is stable even when the Laplacian spectrum is degenerate.

\subsubsection{The heat kernel signature} We consider a graph $G(V, E)$ with graph Laplacian $L$, and assume the eigenvalues $\lambda_i$ and eigenvectors $\phi_i$ of $L$ are given. The {\it heat kernel signature} of a point $x$ is defined in~\cite{Sun2009ADiffusion} as
    \[
        \mathsf{hks}(x,t) = \sum\limits_{i = 0}^{n} e^{-\lambda_i t} \phi_i^2(x).
    \]
    The heat kernel signature is related to our signature power spectrum as the expectation of the exponential:
    \begin{equation} \label{eq:heatkernel_exp}
        \mathsf{hks}(x,t) = \int e^{-st}\Power \qty[\delta_x](s) \dd{s}.
    \end{equation}
    We can also view the heat kernel signature as a discrete analogue of the restriction of the heat kernel on manifolds~\cref{eq:heatkernelform} to a fixed spatial point and the temporal domain. An important application illustrated in~\cite{Sun2009ADiffusion} is that the heat kernel signature can be used as an informative, multi-scale descriptor for shapes up to isometry. \cite[Proposition 4]{Sun2009ADiffusion} establishes the stability of the heat kernel approximation by bounding errors on matrix exponential estimations, and showing that it depends only on the matrix norm of the perturbation to the Laplacian. \cite{Sun2009ADiffusion} also gives an intuitive understanding of the stability of $\mathsf{hks}$, due to the heat kernel being a weighted average over all possible paths between $x$ and $y$ at time $t$ under Brownian motion. As such, it is robust under local perturbations of the surface.

    This stability can also be viewed through the lens of spectral perturbation theory, which describes the variation of eigenvalues and eigenvectors with respect to matrix perturbations. On an eigenvalue level, if $H$ is an $n \times n$ Hermitian matrix with eigenvalues $\lambda_1 \leq \cdots \leq \lambda_n$ and $\widetilde{H}$ is an $n \times n$ Hermitian matrix with eigenvalues $\mu_1 \leq \cdots \leq \mu_n$, then, by Weil's theorem, we know
    \begin{equation*}
        \lambda_i + \mu_1 \leq \sigma_i \leq \lambda_i + \mu_n, \quad i = 1,\, \hdots,\,n,
    \end{equation*}
    where the $\sigma_1 \leq \cdots \leq \sigma_n$ are the eigenvalues of $H + \widetilde{H}$. This result, combined with the fact that, for any Hermitian matrix $H$, $\lambda_{\text{max}}(H) = \Vert H \Vert_2$, yields a result (see~\cite{stewart1991perturbation}) that guarantees that for each $i \in \{1,\, \hdots,\,n\}$, we have
    \begin{equation*}
        |\sigma_i - \lambda_i| \leq \Vert \widetilde{H} \Vert_2,
    \end{equation*}
    that is, the eigenvalues of $H$ and its perturbed version $H + \widetilde{H}$ differ by at most the magnitude of the perturbation. Moreover, the Davis-Kahan-$\sin{\theta}$ \cite{DavisKahan1970}, can be used to bound the distance between subspaces spanned by eigenvectors of a matrix and its perturbed version. This theorem guarantees that the computation of the perturbed eigenvalues and eigenspaces is robust to small perturbations of the data set. This type of argument becomes more complicated for degenerate eigenvalues, which may yield unstable eigenfunctions. A stability argument using matrix perturbation theory that is not susceptible to this instability follows from our main result \Cref{thm:Lipschiz}, which is detailed in \Cref{rmk:signature_stability}.

\subsubsection{Diffusion Distances} In~\cite{Coifman2006DiffusionMaps}, the point signature is defined via the {\it diffusion distance} between two points $x$ and $y$, which measures the intrinsic distance between these points as captured by a diffusion kernel $k_t$, a class of kernels that includes the heat kernel (see \Cref{sec:applications} for further discussion) 
The {diffusion distance} is defined by
    \begin{align*}
        \mathsf{D}_t(x,y) 
        &= \sum\limits_{i = 0}^{n} e^{-\lambda_i t} (\phi_i(x) - \phi_i(y))^2,
    \end{align*}
and we can regard the associated diffusion coordinate 
    \[
        \Phi_t(x) = [e^{-\lambda_1 t/2} \phi_1(x),\,  e^{-\lambda_2t/2}\phi_2(x), \, \hdots,\, e^{-\lambda_n t/2} \phi_n(x)]
    \]
as a signature of the point $x$. It is easily shown that the diffusion coordinates can be used to compute the distance
    \[
        \mathsf{D}_t(x,y) = \Vert \Phi_t(x) - \Phi_t(y)\Vert.
    \]
Observe that the diffusion distance and heat kernel signature are intimately related through the equation 
    \begin{align*}
        \mathsf{D}^2_t(x,y) 
        &= \mathsf{hks}(x,t) + \mathsf{hks}(y,t) - 2 k_t(x,y).
    \end{align*}
Hence, the stability of the diffusion distance can be understood in terms of that of the heat kernel. Similar to the heat kernel signature, the diffusion distance can also be expressed as an expectation with respect to the power spectrum as:
    \[
        \mathsf{D}^2_t(x,y)
        = \int e^{-st}\Power \qty[\delta_x - \delta_y](s) \dd{s}.
    \]
As such, \Cref{thm:Lipschiz} implies that the diffusion distance is stable under perturbations of the input operator via \Cref{rmk:signature_stability}.

    
    

\subsubsection{The Global Point Signature} Another relevant point signature closely related to the heat-kernel signature and the diffusion distance is the {\it global point signature} (GPS)~\cite{Rustamov07}. For a given point $x$, GPS$(x)$ is a vector composed of scaled eigenfunctions of the Laplace-Beltrami operator evaluated at $x$: 
\[
    \mathrm{GPS}(x) = \left[\frac{\phi_1(x)}{\sqrt{\lambda_1}},\, \frac{\phi_2(x)}{\sqrt{\lambda_2}} \hdots, \frac{\phi_n(x)}{\sqrt{\lambda_n}} \right]. 
\]
Some advantages of this signature are invariance under isometric deformations of a shape and robustness to noise since it does not explicitly rely on geodesic distances. However, in contrast to all the signatures considered here, the GPS vector depends on eigenfunctions defined up to a change in sign. As a result, this signature is less stable, and closely spaced eigenvalues can lead to eigenfunction switching.
    
\subsubsection{The Wavelet Signature} The Spectral Graph Wavelet Transform (SGWT)~\cite{Hammond2011WaveletsTheory} of a vertex function $f \in \R^V$ is defined as
    \[
        (T_g^t f)(x) = \sum\limits_{i = 0}^{n} g(t\lambda_i) \phi_i(x)\langle \phi_i, f  \rangle, 
    \]
    where $t$ is the scale parameter, and $g$ is a non-negative real-valued function that behaves as a band-pass filter. For associated wavelets are orthogonal and well-defined, $g$ must have the properties $g(0) = 0$ and $\lim_{\lambda \to \infty} g = 0$. 
    In~\cite{LiChunyuan2013Amdf}, the SGWT was used to define the spectral wavelet signature at a vertex $x$ by setting $f = \delta_x$:
    \begin{align*}
        W_g(x,t) 
        &= \sum\limits_{i = 0}^{n}  g(t\lambda_i)\phi^2_i(x) = \int g(ts) \Power \qty[\delta_x](s) \dd{s},
    \end{align*}
    which we can once again express as an expectation of the function $g(ts)$ with respect to the power spectrum at the vertex $x$. The wavelet signature allows the analysis and design of efficient shape signatures for nonrigid 3D shape retrieval. The optimization process for the algorithm proposed in~\cite{Hammond2011WaveletsTheory} depends on parametrizing the wavelet function (e.g., basis function, their number, as well as other parameters from a Chebyshev polynomial approximation). 
    Once again,  \Cref{thm:Lipschiz} implies that spectral wavelet signatures are stable under perturbations of the input operator via \Cref{rmk:signature_stability}.

\subsection{Representing Power Spectra using Quantiles}
The power spectra $\Power\qty[f_x]$ and $\Power\qty[f_{x,y}]$ link the global properties of the graph described by its eigendecomposition to individual and pairs of vertices. Our primary objective is to utilize power spectra to generate graph features for machine learning applications. However, since power spectra are probability measures on $\R$, they cannot, in principle, be directly parametrized in a finite-dimensional space. To address this, we propose the following approach to summarize the power spectra as finite-dimensional vectors.

The method we develop uses quantiles of the distribution as a finite-dimensional summary of the distribution. The quantile function at $t \in [0,1]$ is given by 
\begin{equation}\label{eq:quantilesdefinition}
    q(t) =  \inf \sett{x \in \R}{ F(t) \geq x},
\end{equation}
where $F: \R \to [0,1]$ is the cumulative distribution function of the distribution. By sampling the quantile function at fixed values $t_1, \ldots, t_n$, we obtain a finite-dimensional summary of the underlying distribution, in this case, the power spectra. These sampled quantiles are in effect a sorted list of the eigenvalues of the Laplacian, whose frequencies of appearance in the list are proportional to the probability accorded by the power spectra in the limit where infinitely many quantiles are sampled.

The $L_p$-norm of the distance between percentiles also approximates the Wasserstein distance between the two power spectra. For two probability distributions $\mu_1$ and $\mu_2$ on $\R$, the $p$-Wasserstein distance is given by 
\begin{equation}\label{eqn:Wasserstein}
    W_p(\mu_1, \mu_2) = \qty(\int \abs{q_1(t) - q_2(t)}^p \dd{t})^{1/p}.
\end{equation}

We selected the Wasserstein distance as our metric due to its effectiveness and stability in matching distributions. Fundamentally, the Wasserstein distance is the minimal cost required to transport the mass of a probability distribution $\mu_1$ to transform it into another mass of a probability distribution $\mu_2$. Unlike other metrics such as total variation, it accounts for the underlying geometry of the space and offers a more informative summary of the distribution. For example, as shown in \cite{Kolourietal2017}, distribution paths (geodesics) that interpolate between two distributions better preserve structural integrity when constructed using the Wasserstein distance rather than the $L^2$-distance. Furthermore, the Wasserstein distance is more robust to small perturbations in the distribution function. In what follows, we conduct a perturbation analysis of our power spectra features using properties of the Wasserstein distance and classical spectral perturbation theory for the graph Laplace operator.


\section{Stability of Power Spectra}\label{sec:perturbationtheory}

In this section, we demonstrate that our signature and the associated quantiles are stable vertex features. To do this, we show that the Wasserstein distance, \cref{eqn:Wasserstein}, is robust under small perturbations of the Laplacian; in particular, \Cref{thm:Lipschiz} states that the Wasserstein distance perturbation is simply proportional to the magnitude of the Laplacian perturbation. More notably, our result holds regardless of whether the Laplacian is degenerate, setting it apart from other signature stability results.
This implies that the Wasserstein distance captures an intrinsic property of our signature and is the appropriate measure of distance between our power spectra.

Without further ado, we state the main result of this section. 

\begin{theorem}\label{thm:Lipschiz}
    Let $f \in \C^n$ be a unit vector, and denote by $\mu^H_f$ the probability measure of the power spectrum of $f$ under a Hermitian matrix $H \in \C^{n \times n}$. Then, for two Hermitian matrices $H$ and $H'$, we have
    \begin{equation}\label{eq:Lipschiz}
        W_1(\mu_f^{H}, \mu_f^{H'}) \leq n\|H - H'\|_2.
    \end{equation}
\end{theorem}

Before arriving at this theorem, we first establish the convergence in the Wasserstein metric $W_p$ for $p \in [1,\infty)$ by analyzing weak convergence properties of the power spectra. While this convergence result demonstrates the continuity of our signature, it does not provide the convergence rate in~\cref{eq:Lipschiz}. To obtain the latter, we make extensive use of matrix perturbation theory as developed in~\cite{Kato1995PerturbationOperators}, along with a suite of propositions and perturbation lemmas. These lemmas are detailed in \Cref{app:perturbationlemmas}. Note that the arguments in~\Cref{subsec:weakconvetc,subsec:convergencerates} are independent of each other; therefore, a reader primarily interested in convergence rates may skip \Cref{subsec:weakconvetc}.

\begin{remark} \label{rmk:approx_sym}
Suppose $\Phi_\sigma$ is a permutation matrix associated with a permutation $\sigma$, and $f,g \in \R^n$ are related by a permutation $f = \Phi_\sigma^T g$. Because $\mu_g^H = \mu_{\Phi_\sigma^T g}^{\Phi_\sigma^T H \Phi_\sigma} = \mu_{f}^{\Phi_\sigma^T H \Phi_\sigma}$, we obtain the bound on the power spectrum of $f$ and $g$, as a consequence of \Cref{thm:Lipschiz}:
    \begin{equation*}
        W_1(\mu_f^H, \mu_g^H) = W_1(\mu_f^H, \mu_{f}^{\Phi_\sigma^T H \Phi_\sigma}) \leq n \| H - \Phi_\sigma^T H \Phi_\sigma \|_2
    \end{equation*}
In particular, if we consider the power spectrum signatures at vertices $i$ and $j$, this implies  
   \begin{equation*}
        W_1(\mu_i^H, \mu_j^H)  \leq  n \min_{\sigma \in S_n,\ \sigma(j) = i} \| H - \Phi_\sigma^T H\Phi_\sigma \|_2
    \end{equation*} 
In other words, if we have an approximate symmetry $\sigma$ of the matrix, in the sense that $\| H - \Phi_\sigma^T H\Phi_\sigma \|_2$ is small, then we expect $W_1(\mu_i^H, \mu_j^H)$ to be small as well if $i$ and $j$ are in the orbit of $\sigma$. 
\end{remark}

\begin{remark} \label{rmk:signature_stability}
    Since we have the dual representation of the 1-Wasserstein distance as 
\begin{equation}
    W_1(\mu, \nu) = \sup \{ \int f d(\mu - \nu) \ | \ f \in  C^0(\R) \ \text{and} \ \mathrm{Lip}(f) \leq 1 \},
\end{equation}
we can obtain a stability result for point signatures using the stability result above. For example, the variation of the heat kernel signature~\cref{eq:heatkernel_exp} at a vertex $x$ with respect to a change in the graph Laplacian $L \mapsto  L'$ can be bounded above by
\begin{align*}
     |\mathsf{hks}_L(x,t)  - \mathsf{hks}_{L'}(x,t)| &= \left | \int e^{-ts} \dd{(\mu^L_x} - \dd {\mu^{L'}_x)} \right| \\
     & = t \left | \int \frac{1}{t}e^{-t x}  \dd{(\mu^L_x} - \dd {\mu^{L'}_x)} \right| \\
     &\leq t W_1(\mu^L_x, \mu^{L'}_x) 
     \leq t n \| L - L' \|_2.
\end{align*}
\end{remark}

\subsection{Convergence in $W_p$ via weak convergence of power spectra}\label{subsec:weakconvetc}

In the proposition below, we show that any sequence $L_k$ or Hermitian matrices that converges to some Hermitian matrix $L$ induces a weak convergence of their corresponding power spectra $\mu^{L_k}_f \Rightarrow \mu^{L}_f$. 
 
\begin{proposition} 
    Let $\{L_k\}$ be a sequence of $n \times n$ Hermitian matrices, and $f \in \C^n$ be some fixed vector with unit norm. If the sequence $L_k$ converges to some Hermitian matrix $L_k \to L$ in some matrix norm $\| \cdot \|$, then we have a weak convergence of measures  $\mu^{L_k}_f \Rightarrow \mu^{L}_f$.
\end{proposition}

\begin{proof}  
    Recall that the \emph{characteristic function} $\varphi: \R \to \C$  of a probability measure is its Fourier transform $\xi \mapsto \mathbb{E}[e^{i\xi x}]$; for $\mu^L_f$, its characteristic is given by 
\begin{equation}
        \varphi^L_f(\xi)  = \sum_{\lambda \in \mathrm{spec}(L)}\langle f, P_\lambda f \rangle e^{i\xi \lambda}.
\end{equation}
The characteristic function $\varphi^L_f$ fully determines $\mu^L_f$; we can recover $\mu^L_f$ from  $\varphi^L_f$ by performing an inverse Fourier transform. Since the eigendecomposition of $L$ is given by $L = \sum_{\lambda \in \spec(L)} \lambda P_\lambda$, we can rewrite the characteristic function using the matrix exponential $\exp : \C^{n \times n} \to \mathrm{GL}(n, \C)$:
\begin{align}
    \varphi^L_f(\xi) &= \left\langle f,\left(\sum_{\lambda \in \mathrm{spec}(L)} e^{i\xi \lambda} P_\lambda \right) f\right\rangle = \langle f, \exp({\imath \xi L}) f\rangle
\end{align}
The matrix exponential is continuous with respect to any matrix norm $\| \cdot \|$. Hence, the characteristic function is a continuous function of $L$. This implies that for any $\xi \in \R$,
    \begin{align*}
     L_k \to L \quad \implies \quad  \varphi_k(\xi) =  \langle f, \exp({\imath \xi L_k}) f\rangle \to \langle f, \exp({\imath \xi L}) f\rangle =: \varphi(\xi).
    \end{align*}
    Since we have pointwise convergence $\varphi_k(\xi) \to \varphi(\xi)$, we can apply L\'evy's convergence theorem, which guarantees that since $\varphi$ is the characteristic of $\mu^L_f$, we have a weak convergence $\mu^{L_k}_f \Rightarrow \mu^{L}_f$.
\end{proof}

Weak convergence of probability measures $\mu_k \to \mu$ on $\R$ yields the following two equivalent conditions. First, if $F_k$ and $F$ are the cdf's of $\mu_k$ and $\mu$ respectively, then for any $t \in \R$ where $F$ is continuous,  $\lim_{k \to \infty} F_k(t) = F(t)$. Second, for any bounded, real valued continuous function $g$, the expectations of $g$ under $\mu_k$ converge to the expectation under $\mu$, i.e., $\mathbb{E}_k[g] \to \mathbb{E}[g]$. More importantly, the result above can also be expressed in terms of the Wasserstein distance since, for $\mu_k$ with support contained in a compact subset of $\R$, weak convergence is equivalent to convergence in the Wasserstein metric $W_p$ for $p \in [1,\infty)$~\cite[Theorem 5.10]{Santambrogio2015OptimalMathematicians}.

\begin{proposition} Let $\{L_k\}$ be a sequence of $n \times n$ Hermitian matrices, and assume $f \in \C^n$ be some fixed vector with unit norm. If $\{L_k\}$ converges to some Hermitian matrix $L$ in a matrix norm $\| \cdot \|$ induced by a vector norm, then for $p \in [1, \infty)$, we have  $W_p(\mu^{L_k}_f, \mu^L_f) \to 0$.
\end{proposition}

\begin{proof}
    Suppose $L_k \to L$ in some matrix norm $\| \cdot \|$. Due to the equivalence of matrix norms, we have in particular, $L_k \to L$ in the operator two-norm $\| \cdot \|_2$, and thus $\|L_k \|_2 \to \|L \|_2$. Recall that, for any finite-dimensional Hermitian matrix $M$, the two-norm is the \emph{spectral radius} $\rho(M) < \infty$, that is, the maximum of the absolute value of the eigenvalues of $M$. Thus $\|L_k \|_2 \to \|L \|_2$ implies for any $r > \rho(L)$, there is some $N$ such that for $k \geq N$ the spectrum of $L_k$ is contained in the interval $(-r,r)$. For $L_k$ and $k \geq N$, the probability measures $\mu^{L_k}_f$, which are supported on the eigenvalues of $L_k$, have supports contained in $[-r,r]$. Since $\mu^{L_{j+N}}_f$ is a sequence of probability measures on $[-r,r]$ that weakly converges to $\mu^L_f$, ~\cite[Theorem 5.10]{Santambrogio2015OptimalMathematicians} implies this is equivalent to $W_p(\mu^{L_{j+N}}_f, \mu^L_f) \to 0$. Consequently, $W_p(\mu^{L_{k}}_f, \mu^L_f) \to 0$ for any sequence $L_k \to L$.
\end{proof}

This concludes our stability proof via weak convergence of power spectra. Next, we turn our attention to quantifying this stability.

\subsection{Convergence Rates}\label{subsec:convergencerates}

The next result bounds the convergence rate in $W_1$. Note that, by the equivalence of Wasserstein's distances shown in~\cite[P.161]{Santambrogio2015OptimalMathematicians}, this is sufficient. In other words, we are allowed to focus our perturbation analysis on $W_1$ with the assurance that the conclusion is transferable to other $W_p$ for any $p \in (1,\infty)$. 


For a Hermitian matrix $H$, recall that $\spec (H)$ denotes its set of distinct eigenvalues. We define  $\gamma(H)$ to be the minimum gap between distinct eigenvalues
\begin{align}
    \gamma(H) &:= \min \{|\lambda -\mu|\ : \ \lambda \neq \mu,\ \lambda, \mu \in \spec (H)\}. \label{eq:spectral_gap} 
\end{align}
Note that even if there are eigenvalues with algebraic multiplicity greater than one, we only count distinct eigenvalues in the definition of $\gamma(H)$. 

In what follows, we present some helpful from matrix perturbation theory that will be crucial in establishing the convergence rate. An extensive discussion of these results can be found in~\cite{Kato1995PerturbationOperators}. 
Assume a Hermitian matrix $H$ is linearly perturbed by another Hermitian matrix $\Delta$. This can be expressed by the linear interpolation $H(t)= H + t\Delta$, for $t \in \R$. The first key result states that the eigenvalues and eigenprojections of $H(t)$ are analytic functions of $t$~\cite[Theorem II.6.1]{Kato1995PerturbationOperators} and provide an expression for their Taylor expansions. These will be essential to bounding errors in the perturbed eigenvalues and eigenvectors' first-order approximations. 

\begin{fact}
    Suppose $\lambda_h$ is an eigenvalue of $H$ with multiplicity $m_h$. Then, we have $j= 1,\ldots, m_h$ analytic functions $\lambda_{h,j}$ satisfying  $\lambda_{h,j}(0) = \lambda_h$, such that  $\lambda_{h,j}(t)$ are eigenvalues of $H(t)$. They admit a locally convergent Taylor series at $t= 0$:
\begin{equation}
    \lambda_{h,j}(t) = \lambda_h + \fo{\lambda_{h,j}}t + \lambda_{h,j}^{(2)}t^2 + \cdots.
\end{equation}
Next, consider the $(m_h)$-dimensional subspace corresponding to the direct sum of the eigenspaces of $\lambda_{h,1}(t)$, $\ldots, \lambda_{h,m_h}(t)$.  At $t = 0$, this is simply the eigenspace of $\lambda_h$. The projection $P_h(t)$ onto this subspace admits a locally convergent Taylor series  at $t = 0$, 
\begin{equation} \label{eq:projection_taylor_series}
    P_h(t) = P_h + \fo{P_h}t + P_h^{(2)}t^2 + \cdots.
\end{equation}
If $t < \frac{\gamma(H)}{2\|\Delta \|_2}$, then the Taylor series of $P_h(t)$ in~\cref{eq:projection_taylor_series} converges~\cite[Theorem II.3.9]{Kato1995PerturbationOperators}. 
\end{fact}

Now, we provide some formulas for first-order correction terms. 

\begin{fact} \label{fact:splitting}
    By~\cite[Theorem II.5.4]{Kato1995PerturbationOperators}, the first-order corrections $\fo{\lambda_{h,j}}$ to the eigenvalues are the set of eigenvalues of the matrix product $P_h\Delta P_h$ in the subspace $\mathcal{V}_h = P_h (\mathbb{C}^n)$, while the first order correction $\fo{P_h}$ to the projection is given by
    \begin{align}\label{eq:projection_fo}
        \fo{P_h} &= -(P_h\Delta S_h + S_h \Delta P_h).
    \end{align}
    Here $S_h$ is the \emph{reduced resolvent}~\cite[eq.s I.5.27-29, p.40]{Kato1995PerturbationOperators} of the eigenvalue $\lambda_h$
    \begin{equation} \label{eq:resolvent}
        S_h = \sum_{k \neq h}\frac{1}{\lambda_k - \lambda_h}P_k.
    \end{equation}
\end{fact}
\begin{remark} \label{rmk:projection_stability}
    We note that \Cref{fact:splitting} implies the magnitude of the perturbation on the projection operators depends on the eigengap $\frac{1}{\gamma(H)}$. Despite the contribution of this perturbation to the bound in \Cref{thm:Lipschiz}, the $\frac{1}{\gamma(H)}$ contribution is canceled out in the calculations in the proof of \Cref{prop:local_bound}. We thus obtain an upper bound on the stability of power spectrum signatures that does not depend on the eigengap. 
\end{remark}
We shall denote the projection onto the first $k$ eigenspaces of $H(t)$ as $P_{[k]}(t)$, which admits for $\|\Delta \|_2 |t| < \frac{\gamma(H)}{2}$ a convergent Taylor series 
\begin{align}
    P_{[k]}(t) = \sum_{h=1}^k P_h(t)  = P_{[k]} + \left(\sum_{h=1}^k \fo{P_h} \right)t + \mathcal{O}(t^2) = P_{[k]} + \fo{P_{[k]}}t + \mathcal{O}(t^2).
\end{align}

In \Cref{lem:projection_series_first_k}, we prove that the first-order correction can be expressed as 
\begin{align}\label{eq:first-ordercorrectionformula}
   \fo{P_{[k]}} :=  \sum_{h=1}^k \fo{P_h} =  \sum_{1 \leq h \leq k <j \leq n}\frac{1}{\lambda_h - \lambda_j} (P_h \Delta P_j + P_j \Delta P_h).
\end{align}
We bound the effect of the perturbation on power spectra by considering the perturbation on the support $\lambda$ and projections $\langle f, P_{[k]} f \rangle$ separately. The first lemma below isolates the effects of the changes $\lambda$ on the Wasserstein distance. 
 
\begin{lemma} \label{prop:taylor_bound_a}
   Consider a Hermitian linear perturbation $H(t) = H + t \Delta$. For any $f \in \C^n$ with $\norm{f}_2= 1$, if $t \|\Delta\|_2 < \gamma(H)/2$,
    \begin{equation}\label{eq:firstWassbound}
        W_1\left(\mu^{H}_f,  \mu^{H(t)}_f\right) \leq  \sum_h (\lambda_{h+1} - \lambda_h) \left| \left\langle f, \left( P_{[h]} - P_{[h]}(t) \right)f\right\rangle   \right|  + t\|\Delta \|_2. 
    \end{equation}
\end{lemma}


\begin{proof}
    Let $\lambda_h$ be the distinct eigenvalues of $H$, and $P_h(t)$ be the Taylor series on the projection operators in~\cref{eq:projection_taylor_series}. Let us define another matrix $\widetilde{H}(t) = \sum_h \lambda_h P_h(t)$, where the eigenvalues do not depend on $t$ while the projections onto the eigenspaces vary in the same way as those of $H(t)$. 
    By the triangle inequality, we have 
    \begin{align}\label{eq:triangineqwasser}
        W_1\left(\mu^{H}_f,  \mu^{H(t)}_f\right) \leq W_1\left(\mu^{H}_f,  \mu^{\widetilde{H}(t)}_f\right)   + W_1\left(\mu^{\widetilde{H}(t)}_f,  \mu^{H(t)}_f\right),
    \end{align}
    where 
    \begin{align*}
        \mu^{\widetilde{H}(t)}_f &= \sum_k \delta(x - \lambda_k) \langle f, P_k(t) f \rangle, \\
        \mu^{{H}(t)}_f &= \sum_h \sum_{j=1}^{m_h} \delta(x - \lambda_{h,j}(t)) \langle f, P_{h,j}(t) f \rangle.
    \end{align*}
Here we recall the notation in \Cref{fact:splitting}, where $\lambda_h$ is an eigenvalue with multiplicity $m_h$, and $\lambda_{h,j}(t)$ are the perturbations of each $j = 1, \ldots, m_h$ eigenvalues that coincide at $\lambda_{h,j}(0) = \lambda_h$. The operator $P_{h,j}(t)$ denotes the projection onto the eigenspace of $H(t)$ corresponding to the eigenvalue  $\lambda_{h,j}(t)$. By definition, $P_h(t) = \sum_{j=1}^{m_h} P_{h,j}(t)$; thus
    \[
        \sum_{j=1}^{m_h} \langle f, P_{h,j}(t) f \rangle = \langle f, \sum_{j=1}^{m_h} P_{h,j}(t) f \rangle = \langle f, P_h(t) f \rangle. 
    \]
    Let us define a transport plan between from $\mu_{\scriptscriptstyle{f}}^{\scriptscriptstyle{H(t)}}$ to $\mu^{\scriptscriptstyle{\widetilde{H}(t)}}_{\scriptscriptstyle{f}}$. Let $T((h,j), k) =  \langle f, P_{h,j}(t) f \rangle \delta_{h,k}$ denote the amount of measure transported from the Dirac measure supported on $\lambda_{h,j}$ in $\mu_{\scriptscriptstyle{f}}^{\scriptscriptstyle{H(t)}}$ to that supported on $\lambda_k$ in $\mu^{\scriptscriptstyle{\widetilde{H}(t)}}_{\scriptscriptstyle{f}}$. Using the equality above, we see that $T$ satisfies the normalization conditions of a transport plan of discrete measures, from $\mu_{\scriptscriptstyle{f}}^{\scriptscriptstyle{H(t)}}$ with support on $\lambda_{(h,j)}$, to $\mu_{\scriptscriptstyle{f}}^{\scriptscriptstyle{\widetilde{H}(t)}}$ with support on $\lambda_k$:
    \begin{align*}
       \sum_{h} \sum_{j = 1}^{m_h}T((h,j), k) &=  \sum_{h} \sum_{j = 1}^{m_h}\langle f, P_{h,j}(t) f \rangle \delta_{h,k} =  \langle f, P_k(t) f \rangle, \\
       \sum_{k}T((h,j), k) &=  \sum_{k} \langle f, P_{h,j}(t) f \rangle \delta_{h,k} =  \langle f, P_{h,j}(t) f \rangle.
    \end{align*} 
Choosing the transport cost between $s, t \in \R$ to be the 1-cost $c(s,t) = |s -t |$, the total transport cost of the transport plan $T$ from $\mu_{\scriptscriptstyle{f}}^{\scriptscriptstyle{H(t)}}$ to $\mu^{\scriptscriptstyle{\widetilde{H}(t)}}_{\scriptscriptstyle{f}}$ is given by 
\begin{align*}
        \sum_{h} \sum_{j = 1}^{m_h} \sum_k |\lambda_{h,j}(t) - \lambda_k|  T((h,j),k)   &=  \sum_{h} \sum_{j = 1}^{m_h} \sum_k |\lambda_{h,j}(t) - \lambda_k|\langle f, P_{h,j}(t) f \rangle \delta_{h,k} \\
        &=  \sum_{h} \sum_{j = 1}^{m_h}  |\lambda_{h,j}(t) - \lambda_h|\langle f, P_{h,j}(t) f \rangle .
    \end{align*}
    
   Since $W_1(\mu^{\scriptscriptstyle{\widetilde{H}(t)}}_{\scriptscriptstyle{f}}, \mu^{\scriptscriptstyle{H(t)}}_{\scriptscriptstyle{f}})$ is the infimum of the transport cost over all possible transport plans, we use the transport cost above to provide an upper bound:
    \begin{align*}
        W_1\left(\mu^{\widetilde{H}(t)}_f,  \mu^{H(t)}_f\right) 
        &\leq    \sum_{h} \sum_{j = 1}^{m_h}  |\lambda_{h,j}(t) - \lambda_h|\langle f, P_{h,j}(t) f \rangle .
    \end{align*}
    However, since the size of the perturbation $t \|\Delta\|_2 < \gamma(H)/2$ is bounded above by half the eigengap of $H$, we can bound the perturbation size on the eigenvalues by the perturbation size on the matrix  $|\lambda_{h,j}(t) - \lambda_h| \leq t \|\Delta\|_2$. Substituting this bound into the latest inequality and applying some inner product properties gives us
    \begin{align*}
        W_1\left(\mu^{\widetilde{H}(t)}_f,  \mu^{H(t)}_f\right) 
        &\leq    \sum_{h} \sum_{j = 1}^{m_h}  |\lambda_{h,j}(t) - \lambda_h|\langle f, P_{h,j}(t) f \rangle\\
        &\leq  t\norm{\Delta}_2  \left \langle f,\sum_h \sum_{j = 1}^{m_h} P_{h,j}(t) f\right\rangle \\
        &= t\norm{\Delta}_2  \left \langle f,f\right\rangle = t\norm{\Delta}_2 .
    \end{align*}
    In the penultimate equality we used the fact that the total projection $\sum_h \sum_{j = 1}^{m_h} P_{h,j}(t)$ onto the eigenspaces of the Hermitian matrix $H + t\Delta$ is the identity, and in the last line, we used the fact that $f$ has unit norm. 
    
    We now turn to developing an expression for the first term $W_1(\mu^{\scriptscriptstyle{\widetilde{H}(t)}}_{\scriptscriptstyle{f}}, \mu^{\scriptscriptstyle{H(t)}}_{\scriptscriptstyle{f}})$ of~\cref{eq:triangineqwasser}. We first observe that $\mu^{\scriptscriptstyle{H(t)}}_{\scriptscriptstyle{f}}$ and $\mu^{\scriptscriptstyle{\widetilde{H}(t)}}_{\scriptscriptstyle{f}}$ share the same support. If we consider two discrete measures $\mu(x) = \sum_h \mu_h \delta(x - \lambda_h)$ and $\nu(x) = \sum_h \nu_h \delta(x - \lambda_h)$, then 
    \begin{align*}
        W_1(\mu, \nu) 
        &= \int_{-\infty}^\infty |F_\mu(x) - F_{\nu}(x)|\dd{x}
        = \sum_h \int_{\lambda_h}^{\lambda_{h+1}} |F_\mu(x) - F_{\nu}(x)|\dd{x} \\
        &= \sum_h (\lambda_{h+1} - \lambda_h) \left| \sum_{j \leq h} \mu_j  - \sum_{j \leq h}  \nu_j \right|.
    \end{align*}
    Employing this result for $\mu_f^H = \sum_h \lambda_h \langle f, P_hf \rangle $ and $\mu_f^{\widetilde{H}(t)} = \sum_h \lambda_h \langle f, P_h(t)f \rangle $, we get
    \begin{align*}
            W_1\left(\mu^{H}_f,  \mu^{\widetilde{H}(t)}_f\right)
            &= \sum_h (\lambda_{h+1} - \lambda_h) \left| \sum_{j \leq h} \langle f, P_j f\rangle   - \sum_{j \leq h}  \langle f, P_j(t) f\rangle   \right| \\
            &= \sum_h (\lambda_{h+1} - \lambda_h) \left| \left\langle f, \left(\sum_{j \leq h} P_j - \sum_{j \leq h}P_j(t) \right)f\right\rangle   \right| \\
            &= \sum_h (\lambda_{h+1} - \lambda_h) \left| \left\langle f, \left( P_{[h]} - P_{[h]}(t) \right)f\right\rangle   \right| 
    \end{align*}
    We obtain the stated result by putting together the two individual bounds in the triangle inequality.
\end{proof}
Having accounted for the effect of the eigenvalue perturbations in \Cref{prop:taylor_bound_a}, we bound the contributions due to variations in the projection operator. We refine the bound (\ref{eq:firstWassbound}) by deriving upper estimates for the term 
 \[\sum_h (\lambda_{h+1} - \lambda_h) \left| \left\langle f, \left( P_{[h]} - P_{[h]}(t) \right)f\right\rangle  \right| .  \] 
 In particular, we obtain a leading order expression of this expression for small $t$. 

\begin{proposition}\label{prop:local_bound} Consider again the linear Hermitian perturbation of an $n \times n$ Hermitian matrix $H(t) = H + t \Delta$. Let $\gamma(H)$ be as defined in Equation (\ref{eq:spectral_gap}). If $\|\Delta \|_2  |t| <  \frac{\gamma(H)}{2}$, then 
    \begin{equation}\label{eq:local_bound}
        W_1(\mu_f^{H}, \mu_f^{H(t)}) \leq n\|\Delta \|_2 |t| + \mathcal{O}((t\|\Delta \|_2)^2).
    \end{equation}
\end{proposition}

\begin{proof}
As aforementioned, we build upon the inequality in~\Cref{prop:taylor_bound_a}, focusing on the first term in the bound. Substituting the Taylor expansion for $P_{[h]}(t)$ at $t = 0$, we obtain
\begin{align*}
    \sum_h (\lambda_{h+1} - \lambda_h) \left| \left\langle f, \left( P_{[h]} - P_{[h]}(t) \right)f\right\rangle   \right| &=  \left(\sum_h (\lambda_{h+1} - \lambda_h) \left| \left\langle f,  \fo{P_{[h]}} f\right\rangle   \right|\right) t  + \order{t^2}.
\end{align*}
Then, using the expression for the first order term of the Taylor expansion of $P_{[k]}$ provided in \Cref{lem:projection_series_first_k}, we get
\begin{align*}
    \sum_k (\lambda_{k+1} - \lambda_k) \left| \left\langle f,  \fo{P_{[k]}} f\right\rangle   \right| &= \sum_k (\lambda_{k+1} - \lambda_k)\sum_{ h \leq k <j }\frac{1}{\lambda_h - \lambda_j} \left| \left\langle f,(P_h \Delta P_j + P_j \Delta P_h)f\right\rangle  \right| \\
    &= \sum_{h <j}(\lambda_{h} - \lambda_j) \frac{1}{\lambda_h - \lambda_j} \left| \left\langle f,(P_h \Delta P_j + P_j \Delta P_h)f\right\rangle  \right| \\
    &=\sum_{h <j } \left| \left\langle f,(P_h \Delta P_j + P_j \Delta P_h)f\right\rangle  \right| \\
    &\leq \sum_{h <j} \left| \left\langle f,(P_h \Delta P_j )f\right\rangle  \right| + \sum_{h <j } \left| \left\langle f,(P_j \Delta P_h)f\right\rangle  \right| \\
    &= \sum_{j \neq h} \left| \left\langle f,(P_h \Delta P_j )f\right\rangle  
    \right|.
\end{align*}
In the second equality, we made use of a technical observation to simplify the summation, which we derive in \Cref{lem:cancellation_first_order}.  In the final equality, we used the fact that $P_h$ are orthogonal projections and are thus self-adjoint, and reordered the indices. Thus, we obtain
\[
    \sum_k (\lambda_{k+1} - \lambda_k) \left| \left\langle f,  \fo{P_{[k]}} f\right\rangle   \right| 
    = \sum_{j \neq h} \left|  (P_h f)^\ast \Delta  (P_j f) \right|. 
\]
 We now bound the term we have obtained. Letting $c_h = \norm{P_h f}_2$, we first observe that $\sum_h c_h^2 = 1$ since the $P_h$'s are complete orthogonal projections. Letting $c$ be the vector of the norm of projections, we have 
\begin{align*}
    \sup_{\norm{f}_2 = 1}\sum_{j \neq h} \left|  (P_h f)^\ast \Delta  (P_j f) \right | & \leq  \sup_{\norm{f}_2 = 1}\sum_{j \neq h} \norm{P_hf}_2 \norm{\Delta}_2 \norm{P_j f}_2 =  \left(\sup_{\substack{\norm{c}_2 = 1 \\ c \geq 0}}\sum_{j \neq h}c_h c_j \right) \norm{\Delta}_2
\end{align*}
Supposing there are $m \leq n$ distinct eigenvalues, then the term in the bracket is bounded above by $n -1$, as we can write
\begin{align*}
\sum_{j \neq h}c_h c_j 
&= \sum_{j,h}c_h c_j  - \sum_j c_j^2  
=  \left(\sum_{j}c_j\right)^2  - \sum_j c_j^2 \\
&\leq \left(\sum_{j}c_j^2\right) \left(\sum_{j} 1 \right)  - \sum_j c_j^2 = m -1 
\end{align*}
using the Cauchy-Schwartz inequality and the observation on the sums of $c_j^2$. It follows that
\begin{align*}
\sum_{j \neq h}c_h c_j \leq n-1.
\end{align*}
Putting everything together, we have 
\begin{align*}
    \sum_k (\lambda_{k+1} - \lambda_k) \left| \left\langle f,  \fo{P_{[k]}} f\right\rangle   \right| \leq(n-1)\norm{\Delta}_2. 
\end{align*}
Substituting this into the bound obtain in \Cref{prop:taylor_bound_a}, we have
\begin{align*}
        W_1\left(\mu^{H}_f,  \mu^{H(t)}_f\right) 
        &\leq |t|\|\Delta \|_2 + \sum_h (\lambda_{h+1} - \lambda_h) \left| \left\langle f, \left( P_{[h]} - P_{[h]}(t) \right)f\right\rangle   \right|  \\
        &= |t|\|\Delta \|_2 + \sum_h (\lambda_{h+1} - \lambda_h) \left| \left\langle f, \left( \sum_{p =1}^\infty t^p P_{[h]}^{(p)}\right)f\right\rangle   \right|  \\
        &\leq |t|\|\Delta \|_2 + \sum_h (\lambda_{h+1} - \lambda_h) \left( \left| t \left\langle f, P_{[h]}^{(1)}f\right\rangle   \right| + \left|  \sum_{p =2}^\infty t^p \left\langle f, P_{[h]}^{(p)}f\right\rangle   \right|  \right) \\
        &\leq  |t|\|\Delta \|_2+ |t|\sum_h (\lambda_{h+1} - \lambda_h) \left| \left\langle f, \fo{P_{[h]}}f\right\rangle   \right|   + t^2 \sum_h (\lambda_{h+1} - \lambda_h) \left|  \sum_{p =0}^\infty t^p \left\langle f, P_{[h]}^{(p+2)}f\right\rangle   \right|   \\
        &\leq |t| \norm{\Delta}_2 + |t|(n-1) \norm{\Delta}_2  + \order{t^2} \\
        &= n|t| \norm{\Delta}_2 + \order{t^2}.
    \end{align*}
\end{proof}

Finally, we prove our main result, \Cref{thm:Lipschiz} by applying \Cref{lem:local_to_global_bound} to extend the local convergence bound \Cref{prop:local_bound} to a global Lipschitz continuity result. 

\begin{proof}(Theorem \ref{thm:Lipschiz})
    Consider the $\C$-vector subspace of $n \times n$ Hermitian matrices. We equip this vector space with the metric induced by the operator two-norm. Within the open ball around the Hermitian matrix $H$ with radius $\gamma(H)/2$, we have 
    \begin{align*}
         W_1(\mu_f^{H}, \mu_f^{H'}) \leq n \norm{H - H'}_2 + \mathcal{O}(\norm{H - H'}_2^2)
    \end{align*}
    due to \Cref{prop:local_bound}. Thus \Cref{lem:local_to_global_bound} implies 
    \begin{align*}
         W_1(\mu_f^{H}, \mu_f^{H'}) \leq n \norm{H - H'}_2.
    \end{align*}
    for any pair of Hermitian matrices $H$ and $H'$.
\end{proof}


\section{Applications}\label{sec:applications}

We consider applications of power spectrum signatures as vertex features to unsupervised and supervised learning problems. In~\Cref{subsec:unsupervisedlearning}, we consider applications to point-cloud clustering, where we use the power spectrum signature at points to identify points that are related by approximate isometries. In~\Cref{subsec:graphregression}, we consider its application to graph regression problems, where we see that including power spectrum signatures of vertices as additional features improves the performance of graph neural network models on benchmark problems.

\subsection{Unsupervised learning: shape-oriented clustering in point cloud data}\label{subsec:unsupervisedlearning}

We consider the application of our method in the context of unsupervised learning on point-cloud data. We seek to group together points that characterize similar structures in point cloud data. 

Given a point cloud $\X \subset \mathbb{R}^n$, we characterize the geometry of the point cloud using the \emph{diffusion operators} on point clouds described by~\cite{Coifman2006DiffusionMaps}, which describes the geometry of the point cloud in a way that is invariant with respect to isometric transformations on the point cloud, such as rotations and reflections. Given a matrix $S$ associated with the diffusion operator, we consider the power spectrum signature $\mu^S_x$ of each point $x \in \X$. We say points are \emph{similar in spectra} if their power spectrum signatures are close in Wasserstein distance. 

We give an overview of how power spectrum signature $\mu^S_x$ are computed. Given some finite point cloud $\mathfrak{X} \subset \mathbb{R}^n$, we consider diffusion operators based on the Gaussian kernel described in~\cref{eq:rbf}. The associated matrices (given in~\cref{eq:diffmap_a}) are parametrized by two hyper-parameters $\alpha, \epsilon$, which we describe below. In our experiments, we consider $\alpha = \frac{1}{2}$.  We then produce power spectra using the following pipeline:
\begin{enumerate}
    \item Fixing $\epsilon$ and $\alpha$, compute the matrix $S$ (\cref{eq:diffmap_a}) from the distance matrix on the point-cloud.
    \item For each point $x \in \X$, compute its power spectrum signature  $\mu^S_x$.
    \item Represent the power spectra signature at each point as quantile vectors. This assigns a vector in Euclidean space to each point in the point cloud.
\end{enumerate}
Having obtained the collection of quantile vectors, we perform PCA to learn features from the quantile vectors, and also perform clustering to identify groups of points that are close under Euclidean distances between the quantile vectors. Since the Euclidean distance between quantile vectors approximates the 2-Wasserstein distance between measures, points with similar quantile vectors have similar power spectra. In the experiments below, we take each quantile vector to be a 1000-dimension-long vector. The clustering algorithm used is DBSCAN. 

Before we give a geometric motivation for studying points that are similar in spectra, we give the technical background with a summary of the construction of diffusion operators, detailed in~\cite{Coifman2006DiffusionMaps}. Using a positive semi-definite kernel $k: X \times X \to [0, \infty)$ on a probability measure space $(X,\mu)$, we first define a weight on each point $\omega(x) = \int k(x,y) d\mu(y)$, and form a new kernel 
\begin{equation}
    \tilde{k}(x,y) = \frac{k(x,y)}{\omega(x)^\alpha \omega(y)^\alpha}
\end{equation}
parametrized by $\alpha \in [0,1]$. 
We then obtain a diffusion operator $\mathcal{A} f(x) = \int A(x,y) f(y) d\mu(y)$, where
\begin{align}
    A(x,y) &= \frac{\tilde{k}(x,y)}{\nu(x)} \quad \text{and}\\
    \nu(x) &= \int \tilde{k}(x,y) d\mu(y).
\end{align}
We then consider the symmetric function 
\begin{equation} \label{eq:diffmap_a}
S(x,y) = A(x,y) \sqrt{\frac{\nu(x)}{\nu(y)}} = \frac{\tilde{k}(x,y)}{\sqrt{\nu(x) \nu(y)}},
\end{equation}
and note that the operator $\mathcal{S} f (x) = \int S(x,y) f(y) d\mu(y)$ and the $\mathcal{A}$ have the same eigenvalues: 
\begin{equation*}
    \mathcal{A} f = \lambda f \iff  \mathcal{S} (\sqrt{\nu} f) = \lambda (\sqrt{\nu} f). 
\end{equation*}
Since $\mathcal{A}$ is a Markov operator, we have $|\lambda | \leq 1$. We note that the eigenfunctions of $\mathcal{A}$ with eigenvalue $\lambda = 1$ correspond to the stationary distributions of the Markov process, as they satisfy $\int A f d \mu = f$. Furthermore, because $S$ is constructed from a positive semi-definite kernel $k$ in the manner described by~\cref{eq:diffmap_a}, $S$ itself is also positive semi-definite, and as such the eigenvalues are all real and non-negative. Hence the eigenvalues of $\mathcal{S}$ and $\mathcal{A}$ lie between $[0,1]$.

In~\cite{Coifman2006DiffusionMaps}, it was shown that if $\X$ is a Riemannian manifold of $\mathbb{R}^n$, and 
\begin{equation} \label{eq:rbf}
    k(x,y) = \exp\left(-\frac{\|x-y\|}{2\epsilon^2}^2\right),
\end{equation}
then in the asymptotic limit $\epsilon \to 0$, the operator $\mathcal{A}$ recovers diffusion by the Laplace-Beltrami operator and a Fokker-Planck operator respectively for $\alpha = 1, \frac{1}{2}$ with order $\epsilon$ convergence rate. 

Consider then the case where $\X$ is a finite set. If there is a permutation matrix $P$ such that $\|S - P^T S P \|_2 $ is small, then by virtue of \Cref{rmk:approx_sym}, we expect the power spectrum of points related by such a permutation to be close under Wasserstein distances. As such, points in the same orbit of an approximate isometry are similar in spectra. 

In the following datasets, we show that the power spectra of indicator functions with respect to diffusion map operators can distinguish structurally different features in a point cloud, and gather together points that are related by approximate global symmetries of the underlying shape of the point cloud. 

\subsubsection{Dataset description}

\paragraph{\emph{Torus}}

We sample 5000 points uniformly on a torus embedded in three dimensions. We choose a torus with major radius (distance from center of torus to center of tube) $R = 1$ and minor radius $r = \frac{1}{4}$. Note that the embedding has a cylindrical symmetry. 

We consider diffusion map matrix $S$, varying  $\epsilon = 0.5, 1.0, 1.5$ while fixing $\alpha = \frac{1}{2}$.

\begin{figure}[h]
    \centering
    \includegraphics[width=0.975\linewidth]{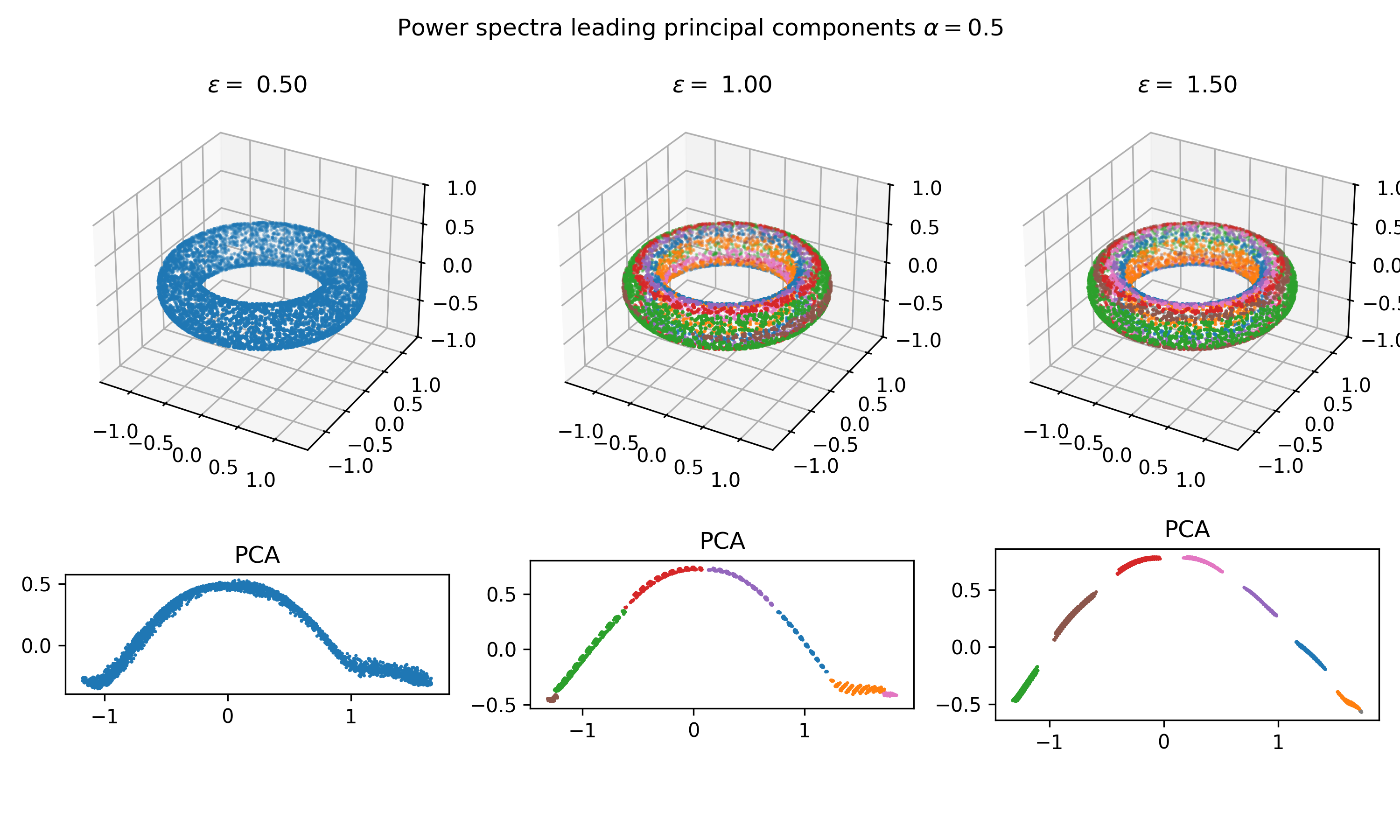}
    \caption{\emph{Torus}. We use PCA to obtain a low-dimensional projection of the quantile vectors of power spectra. We vary the matrix $S_{\alpha, \epsilon}$ from $\epsilon = 0.5, 1.0, 1.5$ while fixing  $\alpha = \frac{1}{2}$. We use the DBSCAN algorithm to cluster the quantile vectors, and color the points by their cluster affiliation. Note DBSCAN groups the whole torus into one cluster.}
    \label{fig:Torus_PS_clusters}
\end{figure}

\paragraph{\emph{Cyclo-Octane Conformations}}
\begin{figure}[h]
    \centering
    \includegraphics[width=0.975\linewidth]{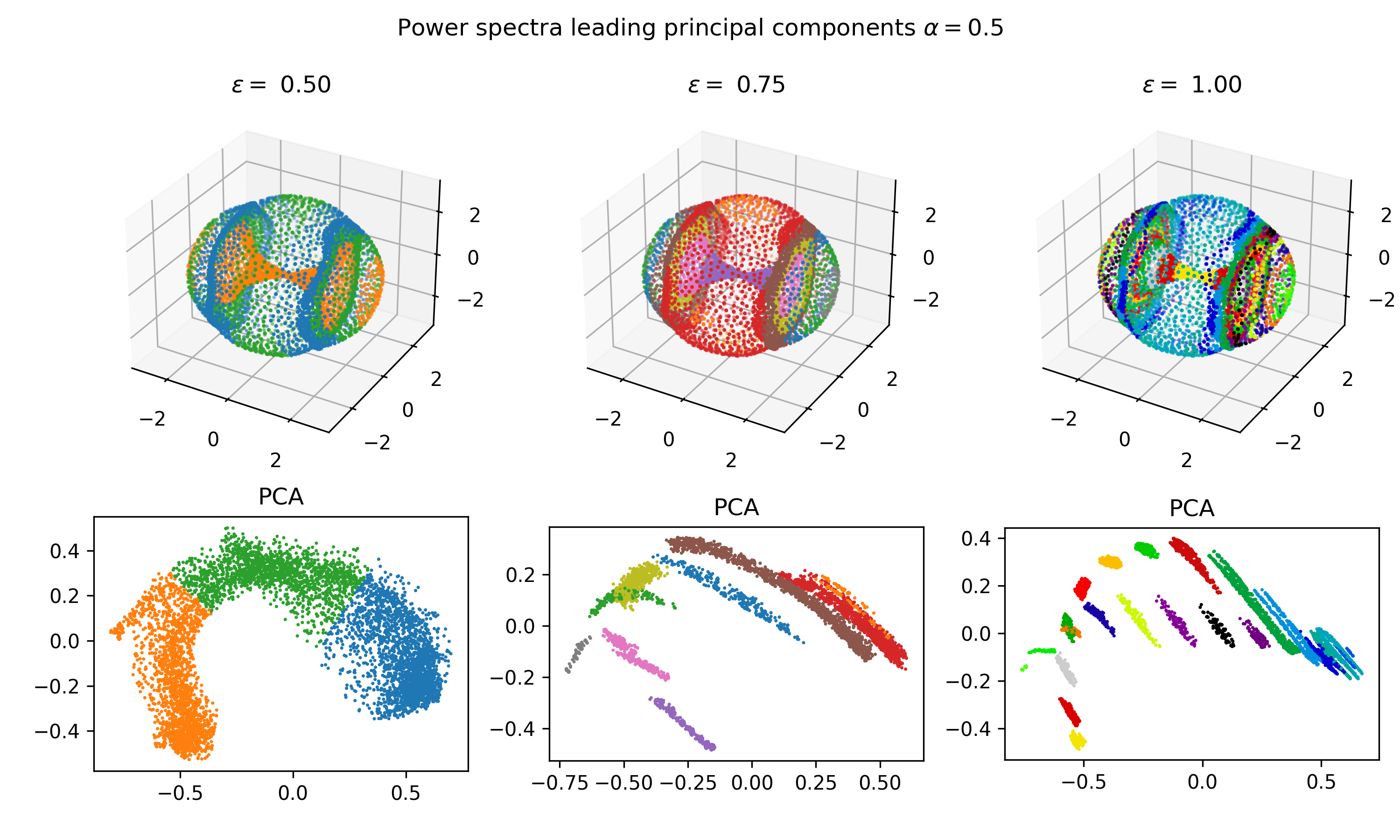}
    \caption{\emph{Cyclo-octane}. We use PCA to obtain a low-dimensional projection of the quantile vectors of power spectra. We vary the matrix $S_{\alpha, \epsilon}$ from $\epsilon = 0.5, 0.75, 1.0$ while fixing  $\alpha = 1/2$. We use the DBSCAN algorithm to cluster the quantile vectors, and color the points by their cluster affiliation. Note that different clusters may appear to overlap in the PCA visualization as the clustering is  }
    \label{fig:CO_PS_clusters}
\end{figure}

We apply our technique to examine the conformation space of the cyclo-octane molecule $C_8H_{16}$. Cyclo-octane is a hydrocarbon chain in which the carbon atoms are arranged in an octagon. We describe the conformation of the molecule by the collective positions of each carbon atom in $\mathbb{R}^3$. This sends each conformation of cyclo-octane to a point in $\mathbb{R}^{3 \times 8}$. We use the subset of the dataset generated by~\cite{Martin2010TopologyLandscape.,Martin2011Non-manifoldData}, subsampled down to 6040 points (we use the same subset as the one used in~\cite{Stolz2020GeometricData}). Details of how the set of conformations is sampled can be found in~\cite{Martin2011Non-manifoldData}. In~\cite{Martin2010TopologyLandscape.}, they empirically verified that the cyclo-octane conformation space is a stratified space, consisting of a Klein bottle attached to a two-dimensional sphere along two circles. In the 3D Isomap low-dimensional visualizations in our figures, the ``hourglass'' corresponds to the Klein bottle component of the conformation space, which cannot be embedded in $\mathbb{R}^3$. 

We consider diffusion map operators $S$, varying $\epsilon = 0.5$, $0.75$, $1.0$ while fixing $\alpha = 1/2$.

\subsubsection{Results and Discussion}

\begin{figure}[h!]
    \centering
    \includegraphics[width=0.975\linewidth]{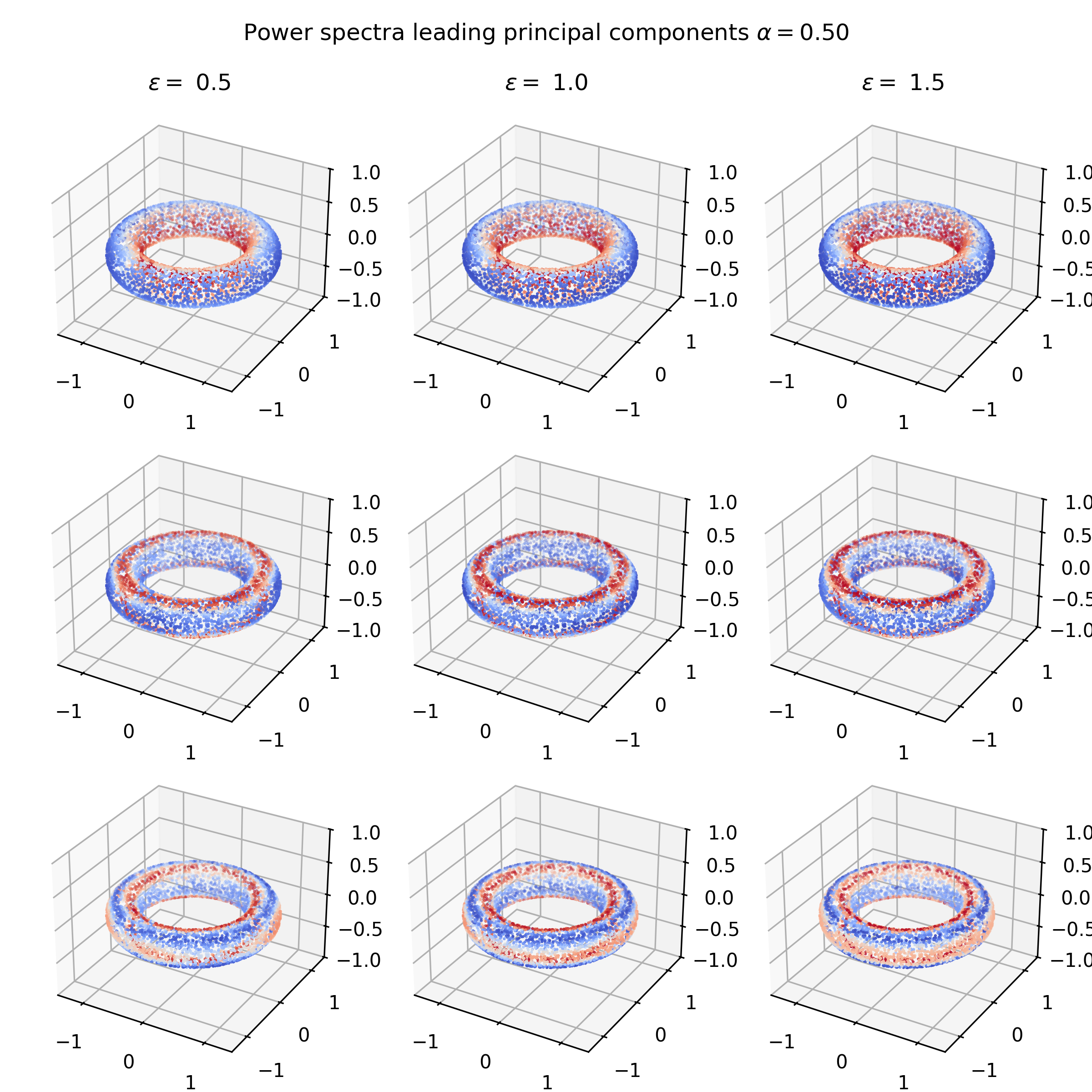}
    \caption{\emph{Torus}. We plot the principal components of the quantile vectors derived from power spectra. Here the quantile vectors are obtained from the vertex power spectra of $S_{\alpha, \epsilon}$ from $\epsilon = 0.5, 1.0, 1.5$ while fixing  $\alpha = \frac{1}{2}$. }
    \label{fig:Torus_PS_pca}
\end{figure}

\begin{figure}[h!]
    \centering
    \includegraphics[width=0.975\linewidth]{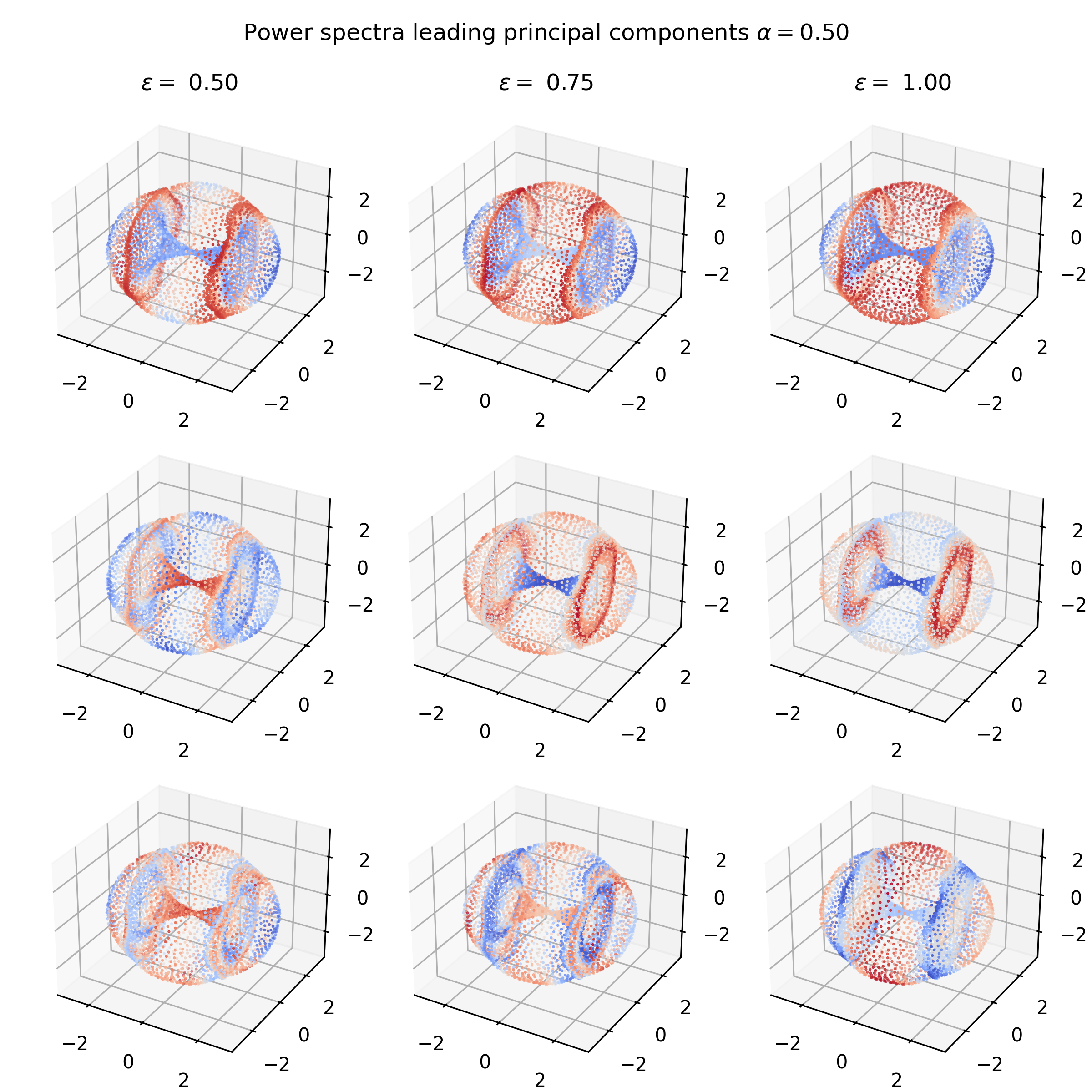}
    \caption{\emph{Cyclo-octane}. We plot the principal components of the quantile vectors derived from power spectra.  Here the quantile vectors are obtained from the vertex power spectra of $S_{\alpha, \epsilon}$ from $\epsilon = 0.5, 0.75, 1.0$ while fixing  $\alpha = \frac{1}{2}$. }
    \label{fig:CO_PS_pca}
\end{figure}

\begin{figure}[h!]
\centering
\includegraphics[width=0.75\linewidth]{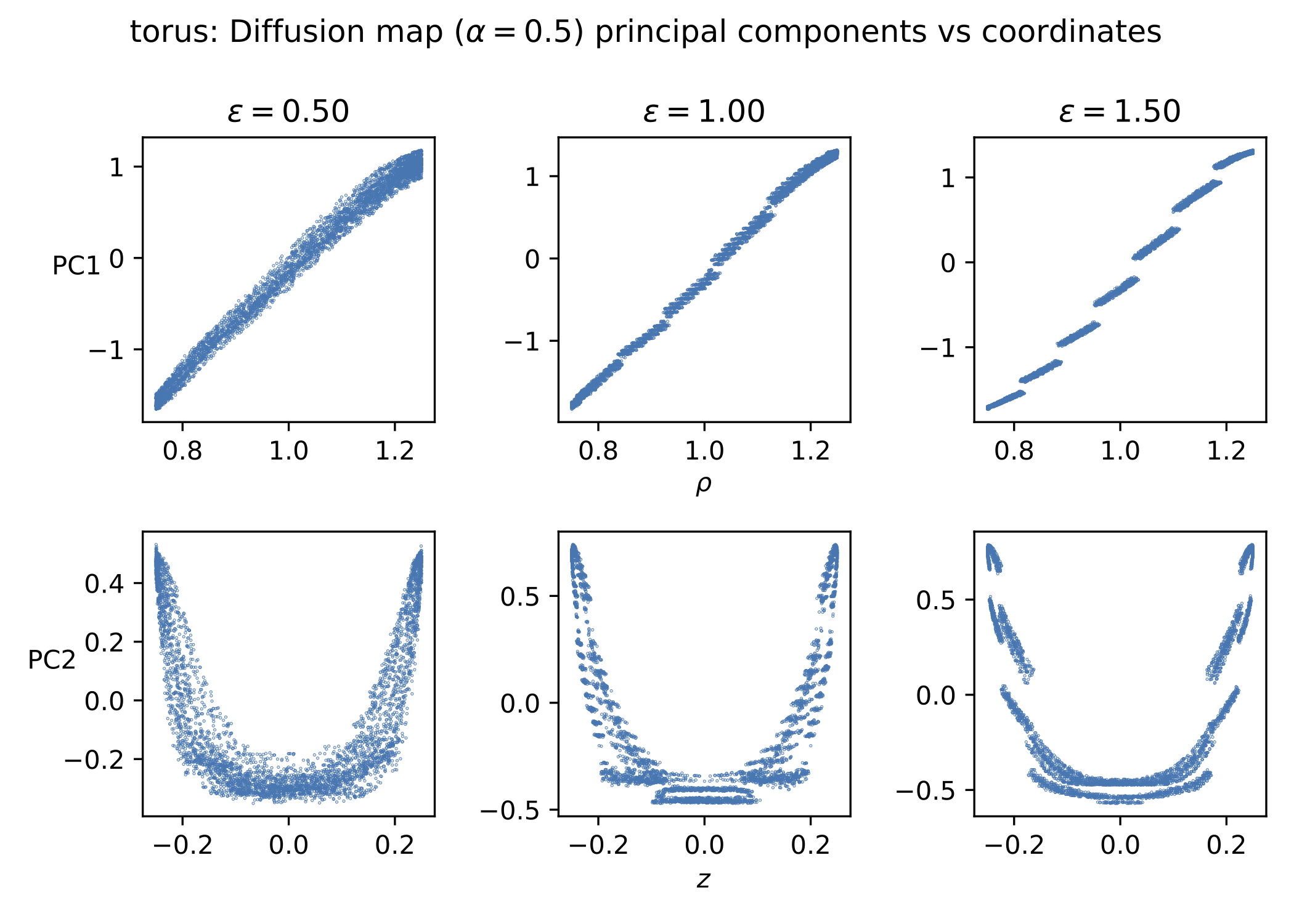}
    \caption{\emph{Torus}. We plot the principal components of the quantile vectors derived from power spectrum signatures at each point, against the cylindrical radius $\rho$ from the $z$-axis, and the $z$-coordinate of the points themselves as they are embedded in $\R^3$.  Here the quantile vectors are obtained from the vertex power spectra of $S_{\alpha, \epsilon}$ from $\epsilon = 0.5, 0.75, 1.0$ while fixing  $\alpha = \frac{1}{2}$. We note that the first principal component PC1 effectively learns $\rho$, while the second principal component PC2 correlates with $|z|$. Note that for the same $z$ value, there are two PC2 branches. This reflects the fact that the level sets of $z$ on the torus consist of two concentric circles of differing radii, and the asymmetry between the inner and outer circles contributes to the different PC2 coordinates. }
    \label{fig:torus_PS_pca_xyz}
\end{figure}

We first contextualize the results of our clustering analysis, illustrated in \Cref{fig:CO_PS_clusters,fig:Torus_PS_clusters}, with a description of the eigenvalues of the matrix $S$ as $\epsilon$ varies. If $\epsilon$ is large relative to the diameter of the dataset, then most kernel entries $k_\epsilon(x,y)$ approach one, and the entries in the matrix $S_{\epsilon, \alpha}$ become increasingly uniform. As such, most eigenvalues become concentrated near zero. This is illustrated in both the torus and cyclo-octane datasets in~\Cref{fig:TorusdiffmapEVs,fig:COdiffmapEVs}. As such, any power spectra is effectively concentrated on a few atoms, and the resultant quantile vectors are effectively parametrized by a few parameters. We also give a geometric interpretation: the Gaussian kernel (\cref{eq:rbf}) is less sensitive to features on a scale $\lesssim \epsilon $, so increasing $\epsilon$ has the effect of obtaining a coarse-grained summary of the geometry in the diffusion operator $S_{\frac{1}{2}, \epsilon}$. Since only large-scale features are effectively retained, the eigendecomposition has less information content, and the spectrum becomes increasingly degenerate. 

This has a downstream effect on the clusters obtained via the quantile vectors. In \Cref{fig:Torus_PS_clusters,fig:CO_PS_clusters}, we visualize the power spectra of each point using the leading two principal components on the quantile vectors. We cluster the quantile vectors using the DBSCAN algorithm. Observe that as we increase $\epsilon$, more distinct clusters emerge in the power spectra. We attribute this transition from continuum-like to discrete collective behavior to the spectrum of $S_{\alpha, \epsilon}$ effectively being supported on fewer eigenvalues as $\epsilon$ increases, which is illustrated in \Cref{fig:CO_PS_clusters}. As such, the power spectra $\sum_{\lambda } \langle f,  P_\lambda f\rangle \delta(t - \lambda)$ of vertex functions are dictated by fewer parameters $ \langle f,  P_\lambda f\rangle $. 

We make some specific comments on the datasets. We first consider the torus. Our embedding of the torus into $\R^3$ is cylindrically symmetric about the $z$-axis, and has a mirror symmetry about the $xy$-plane. The equivalence classes of these symmetries are parametrized simply by the cylindrical radial distance $\rho$ from the $z$-axis. We observe from \Cref{fig:Torus_PS_clusters} that features from the power spectrum signatures respect the cylindrical symmetry of the torus. When clustering points based on their power spectrum signature, points that are in the orbit of the rotation about the $z$-axis are clustered together, with more distinct clusters emerging as $\epsilon$ increases. In addition, the first principal components of the power spectrum signatures correlate with the cylindrical radial distance $\rho$ of the torus embedding, while the second principal correlates with the distance to the $xy$ plane of symmetry (see \Cref{fig:torus_PS_pca_xyz,fig:Torus_PS_pca}). In contrast, the eigenvectors of the diffusion map do not respect the symmetries of the torus (\Cref{fig:TorusdiffmapEVecs}). 

In the case of cyclo-octane, we observe that the clusters identified effectively segment the space into the Klein bottle and the spherical components. In the case of $\epsilon = 0.5$, the quantile vectors are not segmented into clear clusters; nonetheless, the segmentation by the DBSCAN algorithm manages to separate the spherical and Klein bottle components and identify the two singular circles where the Klein bottle and the sphere meet. Note that since the clusters are not clear-cut, the algorithm confuses the extremal spherical caps with the Klein bottle. As we increase $\epsilon$, the finer components emerge, and spherical caps are further segmented by their distance to the attachment circles. The Klein bottle itself is also further segmented by its distance to the attachment circles too. 

In \Cref{fig:COdiffmapEVecs}, we illustrate the three leading eigenvectors for the diffusion map matrix $S_{\alpha, \epsilon}$. We note that apart from the $\lambda = 1$ eigenvector corresponding to the stationary distribution of the diffusion operator, the other eigenvectors do not respect the symmetries of the point cloud.  Depending on the scale parameter $\epsilon$, the $\lambda = 1$ eigenvector distinguishes certain components of the configuration from another: for example, for $\epsilon = 0.5$, the two circles along which the Klein bottle and sphere meet are highlighted as extrema of the eigenvector; however, for $\epsilon = 0.75$, the eigenvector only distinguishes the spherical end caps from the rest of the point cloud, while for $\epsilon = 1.0$, the eigenvector only distinguishes the equatorial belt of the sphere from the rest of the space. In contrast, the power spectra of vertex indicators not only segment different parts of the space effectively, they also respect the (approximate) symmetries of the point cloud. 

\subsection{Graph Regression}\label{subsec:graphregression}
Graph neural networks are a paradigm for machine learning on graphs that relies on aggregating vertex features along the graph to distinguish graphs by their vertex features and structure. We investigate the effect of using the power spectrum signatures at vertices as additional vertex features to enhance the performance of graph neural networks in graph regression tasks. 

In many applications the leading eigenvectors of graph Laplacians are often used to provide additional vertex features; this is often known as a ``positional encoding'' of vertices. We refer the reader to~\cite{Dwivedi2023BenchmarkingNetworks} for a systematic study of incorporating eigenvector features on machine learning problems for graphs. However, the use of eigenvectors as positional encoding suffers from several issues. Eigenvectors are subjective to a choice of sign or basis, and they need not be invariant under the automorphisms of the graph. Furthermore, the eigenvectors are decoupled from their eigenvalues, which remove information about the ``energy'' of the eigenvectors on the graph. Finally, the eigenvectors are often truncated to only the leading few $k$ eigenvectors, since the inclusion of all eigenvectors may lead to additional computational burden and produce vertex features that are mutually orthogonal. These make using Laplacian eigenvectors practically difficult if applied naively as positional encodings. Solutions to the basis choice issue include flipping the sign on the eigenvectors randomly in the training or using sign or basis invariant networks such as SignNet or BasisNet~\cite{Lim2022SignLearning}. Even so, the fully basis invariant BasisNet incurs significant computational burden as the input vector size corresponding to each eigenspace corresponds to the full $n \times n$ projection matrix.  Moreover, due to \Cref{fact:splitting}, the stability of the projection matrices to perturbation depends on the inverse of the eigengap, which can be arbitrarily small (see \Cref{rmk:projection_stability} for detailed comments). 

We investigate how using the power spectrum signature (shortened PSS) of vertices, instead of the Laplacian eigenmap positional encoding, can solve these issues. First, the power spectrum signatures are invariant with respect to choices of eigenbasis. Second, they take the eigenvalues associated to each eigenvector into account. Third, they include information from all eigenspaces, without any truncation. The only loss of information compared to Laplacian eigenmaps is the information attached to the relative sign between entries of the same eigenvector. 

In our experiments, we consider how power spectrum signatures can help lower the error on two graph regression problems. We consider two benchmark graph regression problems ZINC and AQSOL, and compared the performances of three graph neural network architectures GIN~\cite{Xu2018HowNetworks} (the graph isomorphism network), GATv2~\cite{Brody2021HowNetworks} (the graph attention network), and the GatedGCN~\cite{Bresson2017ResidualConvNets} (the gated graph convolution network) with and without the additional power spectrum signatures at vertices. 

\subsubsection{Dataset description}
Both ZINC and AQSOL benchmark datasets consist of molecular graphs with node labels indicating the type of atom, and the edge labels indicating the type of bonds. 

The ZINC benchmark dataset consists of 12,000 molecular graphs. The full dataset introduced by~\cite{Irwin2012ZINC:Biology} contains 250,000 compounds but we use the subset used in~\cite{Fey2019FastGeometric}. The objective of the benchmark problem is to predict the constrained solubility of the molecules. The dataset is split into training, validation, and test sets consisting of 10,000, 1000, and 1000 graphs each. 

The AQSOL benchmark dataset introduced in~\cite{Dwivedi2023BenchmarkingNetworks} consists of 9,982 molecular graphs, extracted from the dataset presented in~\cite{Sorkun2019AqSolDBCompounds}. The objective of the benchmark problem is to predict the aqueous solubility of the molecules. The dataset is split into training, validation, and test sets consisting of 7,831, 966, and 996 graphs each.

The train, validation, and test splits are as used in~\cite{Dwivedi2023BenchmarkingNetworks}, and encoded in the dataset generators in~\cite{Fey2019FastGeometric}.

\subsection{GNN Architecture}
In all experiments, the graphs are passed through four graph convolutional layers before the vertex features are pooled to yield a regression prediction. We consider three distinct graph convolution layers: GIN, GATv2, and GatedGCN. The neural network architectures and training procedures are based on those used in the benchmark experiments conducted in~\cite{Dwivedi2023BenchmarkingNetworks} as a benchmark. However, we differ from their experiments in also incorporating edge features using the implementations of those convolution layers in PyTorch Geometric~\cite{Fey2019FastGeometric}. We give details about the hyperparameters of the models in each experiment, as well as the training procedure in \Cref{app:GNN}.

For both datasets, we use the normalized Laplacian with unit edge weights to generate the vertex power spectra. For AQSOL, we sampled 25 quantiles evenly for each vertex power spectrum as its quantile vector. For ZINC, we sampled 30 quantiles evenly. These are incorporated as additional vertex features on top of the vertex labels. The architecture specifications and training process can be found in~\Cref{app:GNN}.

\subsection{Experimental Results }
In \Cref{tab:AQSOL,tab:ZINC}, we show the mean absolute error training, validation, and test sets with and without the additional power spectrum signatures at vertices. Across all three architectures, and both benchmark problems, we observe that the test error is decreased after the additional features are incorporated. We notice that the ordering of the performance between the three graph neural network datasets does not change even after the addition of the extra features for both problems. We also note that the efficacy of the extra features is certainly problem and dataset-dependent; the addition of such features only decreased the error by 9\% difference for AQSOL, while we see a reduction of 50\% in ZINC. Since the relative decrease of the mean absolute error is consistent across different architectures, we conjecture that spectral features add significant information to
those learned by graph neural networks for predicting the constrained solubility of the compounds in the ZINC dataset, while their contribution is marginal on top of that of the graph neural network for AQSOL. 
\begin{table}[h]
\centering
\begin{tabular}{|l|| c   c |  c  c| c  c|}
\hline
AQSOL & GIN   & with PSS & GAT   & with PSS & GatedGCN   & with PSS  \\ \hline\hline
Train & \texttt{0.693} & \texttt{0.503}& \texttt{0.497} & \texttt{0.302}  & \texttt{0.462} & \texttt{0.384}    \\ \hline
Valid & \texttt{1.073} & \texttt{1.048}     & \texttt{1.171} & \texttt{0.989} & \texttt{\underline{1.001}} & \texttt{\emph{0.914}} \\ \hline\hline
Test  & \texttt{1.602} & \texttt{1.530} & \texttt{1.382} & \texttt{1.145}  & \texttt{\underline{1.231}} & \texttt{\emph{1.117}}\\ \hline
\end{tabular}
\vspace{10pt}
\caption{Comparison of mean absolute error on the AQSOL regression benchmark with and without additional vertex features. We show the validation error on the training, validation, and test subsets. The architecture with the best validation error without PSS is underlined, along with the corresponding test error; the architecture with the best validation error with PSS is italicized, along with the corresponding test error.  \label{tab:AQSOL}}
\end{table}

\begin{table}[h]

\centering
\begin{tabular}{|l|| c   c |  c  c| c  c|}
\hline

ZINC & GIN   & with PSS & GAT   & with PSS & GatedGCN   & with PSS  \\ \hline\hline
Train & \texttt{0.2441} & \texttt{0.0950}    & \texttt{0.2636} & \texttt{0.0853} & \texttt{0.2792} & \texttt{0.0988}     \\ \hline
Valid & \texttt{\underline{0.3044}} & \texttt{\emph{0.1979}}    & \texttt{0.4146} & \texttt{0.2616}  & \texttt{0.3531} & \texttt{0.2591}   \\ \hline\hline
Test  & \texttt{\underline{0.3133}} & \texttt{\emph{0.1577}}    & \texttt{0.4267} & \texttt{0.2431} & \texttt{0.3651} & \texttt{0.2258}     \\ \hline
\end{tabular}
\vspace{10pt}
\caption{Comparison of mean absolute error on the ZINC regression benchmark with and without additional vertex features. The architecture with the best validation error without PSS is underlined, along with the corresponding test error; the architecture with the best validation error with PSS is italicized, along with the corresponding test error.  \label{tab:ZINC}}
\end{table}


\section{Conclusion and Future Work}

In this paper, we introduced the power spectrum signature, a novel point signature based on the spectral decomposition of the graph Laplacian, and demonstrated its many desirable properties. These include permutation invariance and stability under perturbations, particularly when using the Wasserstein distance as a discriminating metric--even in cases of degenerate graph Laplacian spectra. We showcased the effectiveness of power spectra as data features in both unsupervised (clustering on point clouds) and supervised (graph regression) learning. In both settings, we demonstrate the novel utility of our method: in point-cloud clustering, the power spectrum can be used to identify points that are related by approximate isometries and distinguish manifold-like regions from singular regions, while in benchmark graph regression problems, including power spectrum signatures of vertices as additional features demonstrably improves the accuracies of graph neural network models. 

Looking ahead, there are several promising directions for future exploration. From a theoretical point of view, we hope to describe the bound described in \Cref{thm:Lipschiz} in a probabilistic setting. For example, if the perturbation to the matrix is a random Hermitian matrix, then we anticipate that the expected size (or confidence intervals) of the perturbation can be bounded in terms of the parameters and properties of the matrix distribution. Focussing especially on perturbations to graph Laplacians, we hope to obtain similar results on the expected perturbation when the underlying graph is perturbed by random fluctuations in edge weights, or random additions and deletions of edges. Investigating these probabilistic settings allows us to refine our understanding of the stability of power spectrum measures in situations closer to real-world applications. 

There are also further improvements that could be explored to make power spectrum signatures more useful in applications. While we have restricted ourselves to using the power spectrum signature as vertex features, as future work, we can consider the power spectrum of other functions to describe different features of the graphs. Given its role in the injectivity result~\Cref{prop:injectivity}, indicator functions on pairs of vertices can also be used to characterize pairwise relationships between vertices. Other operators, such as the edge Laplacian, can also be used to generate a power spectrum signature associated with edges to further enhance the graph neural networks performance. 

Additionally, although quantile functions have advantageous mathematical properties as they approximate the Wasserstein distance, there may be other more efficient ways to represent features learned from power spectra. In particular, when the size of the graph becomes large and the number of eigenvalues grows along with it, we need more quantile samples to capture the shape of the power spectrum. For example, we can explore using path signatures~\cite{Lyons2007DifferentialPaths,Lyons2022SignatureLearning} to obtain a high-level description of the shape of the quantile function as a path in $\R^2$.

\section*{Acknowledgements}
The authors would like to thank Marco Marletta and Iveta Semoradova for their advice on matrix perturbation theory, and John Harvey and Raphael Winter for their helpful thoughts when we discussed this subject. KYD was partially supported by NSF LEAPS-MPS \#2232344.  KMY is supported by a UKRI Future Leaders Fellowship [grant number MR/W01176X/1; PI: J. Harvey].



\bibliographystyle{plain}
\bibliography{main_ArXiv.bib}



\appendix

\section{Perturbation Analysis Lemmas}\label{app:perturbationlemmas}


\begin{lemma} \label{lem:projection_series_first_k}Consider a Hermitian linear perturbation $H(t) = H + t \Delta$. The first order correction term to $P_{[k]}$, the projection onto the first $k$ eigenspaces of $H(t)$, can be expressed as
\begin{align}\label{eq:}
   \fo{P_{[k]}} :=  \sum_{h=1}^k \fo{P_h} =  \sum_{h \leq k}\sum_{j > k}\frac{1}{\lambda_h - \lambda_j} (P_h \Delta P_j + P_j \Delta P_h).
\end{align}
\end{lemma}
\begin{proof}
   We prove this statement by induction. For the base case $k = 1$, it suffices to directly substitute the definition of the resolvent $S_1$ \cref{eq:resolvent} into the first order perturbation to the projection matrix \cref{eq:projection_fo}: 
   \begin{align*}
       -\fo{P_1} &= P_1\Delta S_1 + S_1 \Delta P_1 =\sum_{j > 1} \frac{1}{\lambda_j - \lambda_1} P_1 \Delta P_j+  \sum_{j > 1} \frac{1}{\lambda_j - \lambda_1} P_j \Delta P_1,
        \intertext{so,}
         \fo{P_1}&= \sum_{h \leq 1} \sum_{j > 1} \frac{1}{\lambda_h - \lambda_j} (P_h \Delta P_j + P_j \Delta P_h). 
   \end{align*}
   Having shown that this is true for the base case, let's assume the statement is true for $\fo{P_k}$ for some $k \geq 1$, and derive the expression for $\fo{P_{[k+1]}} = \fo{P_{[k]}} + \fo{P_{k+1}}$. As a shorthand let us denote $\Delta(i,j) = \Delta(j,i) := P_i \Delta P_j + P_j \Delta P_i$. We first note the following
   \begin{align*}
   \fo{P_{k+1}} &= \sum_{j \neq k+1} \frac{\Delta(k+1,j)}{\lambda_{k+1} - \lambda_j}  =  \sum_{h \leq k} \frac{\Delta(k+1,h)}{\lambda_{k+1} - \lambda_h}  +  \sum_{j > k+1} \frac{\Delta(k+1,j)}{\lambda_{k+1} - \lambda_j}  =  -\sum_{h \leq k} \frac{\Delta(h,k+1)}{\lambda_{h} - \lambda_{k+1}}  +  \sum_{j > k+1} \frac{\Delta(k+1,j)}{\lambda_{k+1} - \lambda_j}. 
   \end{align*}
   Similarly, we can also split up the sum in $\fo{P_{[k]}}$:
   \begin{align*}
       \fo{P_{[k]}} =  \sum_{h \leq k}\sum_{j > k}\frac{\Delta(h,j)}{\lambda_h - \lambda_j}  = \sum_{h \leq k}\frac{\Delta(h,k+1)}{\lambda_h - \lambda_{k+1}} + \sum_{h \leq k}\sum_{j > k+1}\frac{\Delta(h,j)}{\lambda_h - \lambda_j}.
   \end{align*}
   Thus we can write 
   \begin{align*}
       \fo{P_{[k+1]}} &= \fo{P_{[k]}} + \fo{P_{k+1}} = \sum_{h \leq k}\sum_{j > k+1}\frac{\Delta(h,j)}{\lambda_h - \lambda_j} +  \sum_{j > k+1} \frac{\Delta(k+1,j)}{\lambda_{k+1} - \lambda_j} \\
       &= \sum_{h \leq k+1}\sum_{j > k+1}\frac{\Delta(h,j)}{\lambda_h - \lambda_j}.
   \end{align*}
   This completes the inductive step.
\end{proof}

\begin{remark}
    The previous result can also be derived from first principles by taking contour integrals of the resolvent, along the same lines where the perturbative series for the projections onto the individual eigenspaces were derived in~\cite[section II.2.1]{Kato1995PerturbationOperators}. We take the approach above for simplicity of the exposition. 
\end{remark}
\begin{lemma} \label{lem:cancellation_first_order} Let $A \in \C^{n \times n}$, and $\lambda_1, \ldots, \lambda_n$ be real numbers, then
    \begin{align}
        \sum_{k=1}^{n-1} (\lambda_{k+1} - \lambda_k)\sum_{1 \leq h  \leq k < j \leq n} A_{hj} = \sum_{1 \leq h < j \leq n}(\lambda_j - \lambda_h)A_{hj}.
    \end{align}
\end{lemma}
\begin{proof}
    The sum consists of a linear combination of upper triangular entries $A_{hj}$ of the matrix $A$. Each coefficient of $A_{hj}$ consists of the sum over $(\lambda_{k+1}- \lambda_k)$ for every $h \leq k < j$. Therefore, 
    \begin{align*}
        \sum_{k=1}^{n-1} (\lambda_{k+1} - \lambda_k)\sum_{1 \leq h  \leq k < j \leq n} A_{hj} &= \sum_{1 \leq h < j \leq n} A_{hj} \sum_{h  \leq k < j }  (\lambda_{k+1} - \lambda_k) \\
        &= \sum_{1 \leq h < j \leq n}(\lambda_j - \lambda_h)A_{hj}.
    \end{align*}
\end{proof}

\begin{lemma}\label{lem:local_to_global_bound} Let $X$ be a Banach space over the field $\R$ or $\C$, and assume $f: X \to Y$ is a function from $X$ to a metric space $(Y, \dd_Y)$. If for any $x \in X$, there is a sufficiently small $r(x) > 0$ such that for all $x' \in \openball{x}{r(x)}$
\begin{equation}
    \dd_Y(f(x), f(x')) \leq K \norm{x-x'} + \order{\norm{x-x'}^2}, \label{eq:almost_local_lip}
\end{equation}
then $f$ is a $K$-Lipschitz function: i.e., for any pair of points $x,x'$,
\begin{equation}
    \dd_Y(f(x), f(x')) \leq K \norm{x-x'}.
\end{equation}
\end{lemma}

\begin{proof}
For any $\epsilon >0$, the condition expressed in \cref{eq:almost_local_lip} implies for any $x \in X$, there is a sufficiently small radius $\rho(x) \in (0, r(x))$, such that  for all $x' \in \openball{x}{\rho(x)}$
    \begin{equation}
    \dd_Y(f(x), f(x')) \leq (K + \epsilon) \norm{x-x'}. \label{eq:almost_local_lipb}
\end{equation}
We integrate this local condition into a global Lipschitz condition as follows. For any pair of distinct vectors $x_0,x_1 \in X$, we consider its linear interpolation 
    \begin{align*}
        x_t = x_0(1-t) + tx_1, \quad t \in [0,1].
    \end{align*}
We recall that linear interpolation is a geodesic in Banach spaces; that is, for any $t \in [0,1]$, we have
    \begin{align*}
        \norm{x_1 - x_0} = \norm{x_1 - x_t} + \norm{x_t - x_0}.
    \end{align*}
In particular, we note that the distance along linear interpolations is given by 
    \begin{align*}
        \norm{x_t - x_s} = \norm{x_0(1-t) + tx_1 - x_0(1-s) - sx_1} = (t-s) \norm{x_1 - x_0}.
    \end{align*}
Let $\bar{\rho}_t = \frac{\rho(x_t)}{\norm{x_1- x_0}}$. The $\openball{t}{\bar{\rho}_t} \subset [0,1]$ is the set of parameters along the interpolation such that $s \in \openball{t}{\bar{\rho}_t}$  iff $x_s \in \openball{t}{\rho(t)}$. We consider the following open cover of the interval 
    \begin{align*}
        \mathcal{U} = \{\openball{t}{\bar{\rho}_t)} \ | \ t\in [0,1] \}.
    \end{align*}
    Because $[0,1]$ is compact, we can pass to a finite subcover 
    \begin{align*}
        \mathcal{V} = \{\openball{t_i}{\bar{\rho}_{t_i}} \ | \ i =1, \ldots, m  \},
    \end{align*}
     where $0 \leq t_1 < t_2 <  \cdots  < t_m \leq 1$.
    Now consider a choice of parameters 
    \begin{equation*}
        0 = s_0 \leq t_1 \leq s_1 \leq \cdots  \leq s_{m-1} \leq t_m \leq s_m = 1,
    \end{equation*}
    such that $s_i \in \openball{t_{i+1}}{\bar{\rho}_{t_{i+1}}} \cap \openball{t_i}{\bar{\rho}_{t_{i}}}$. It follows then that $\norm{x_{t_{i+1}} - x_{s_{i}}} \leq \rho(x_{t_{i+1}})$ and $\norm{x_{t_{i}} - x_{s_i}} \leq \rho(x_{t_{i}})$. Thus \cref{eq:almost_local_lipb} implies 
    \begin{align*}
        \dd_Y(f(x_{s_i}), f(x_{t_i})) &\leq (K + \epsilon) \norm{x_{s_i}- x_{t_i}} = (K + \epsilon)(s_i - t_i) \norm{x_1 - x_0}; \\
        \dd_Y(f(x_{t_{i+1}}), f(x_{s_i})) &\leq (K + \epsilon) \norm{x_{t_{i+1}} - x_{s_i}} = (K + \epsilon)(t_{i+1} - s_i) \norm{x_1 - x_0}.
    \end{align*}
    We integrate the bounds above by applying the triangle inequality successively:
    \begin{align*}
         \dd_Y(f(x_{0}), f(x_{1}))  &\leq \sum_{i=1}^m \dd_Y(f(x_{s_{i-1}}),f( x_{t_i})) + \dd_Y(f( x_{t_i}), f(x_{s_{i}})) \\
         & \leq \sum_{i=1}^m (K + \epsilon)(t_{i} - s_{i-1}) \norm{x_1 - x_0} +  (K + \epsilon)(s_i - t_i) \norm{x_1 - x_0}  \\
         &=  (K+ \epsilon)\norm{x_1 - x_0}\sum_{i=1}^m s_i - s_{i-1} =  (K+ \epsilon)\norm{x_1 - x_0},
    \end{align*}
    where in the final equality we note that $s_{0} = 0$ and $s_m = 1$. Since we have $\dd_Y(f(x_{0}), f(x_{1})) \leq (K+ \epsilon) \norm{x_1 - x_0}$ for arbitrary $x_0, x_1 \in X$, and any $\epsilon >0$, we must have $\dd_Y(f(x_{0}), f(x_{1})) \leq K\norm{x_1 - x_0}$, as any pair satisfying $\dd_Y(f(x_{0}), f(x_{1})) > K\norm{x_1 - x_0}$ contradicts $\dd_Y(f(x_{0}), f(x_{1})) \leq (K+ \epsilon) \norm{x_1 - x_0}$ for some sufficiently small $\epsilon$.
    
\end{proof}



\section{Diffusion Map Eigenvalues and Eigenvectors}

\begin{figure}[ht]
    \centering
    \includegraphics[width=0.95\linewidth]{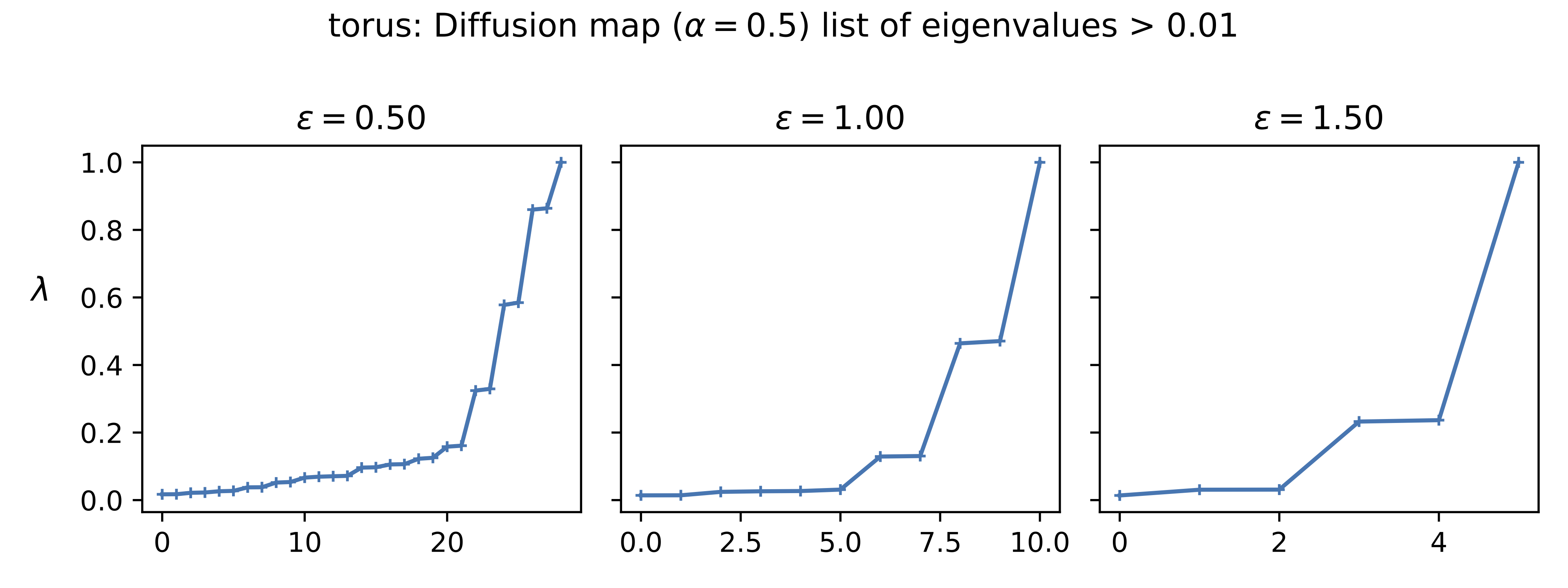}
    \caption{\emph{Torus}. Eigenvalues of the diffusion map matrix $s_{\alpha, \epsilon}$, for $\alpha = 1/2$ and $\epsilon = 0.5, 1.0, 1.5$. We sort the eigenvalues between $[0,1]$ and plot those with $\lambda > 0.01$. 
As our input point cloud has 5000 points, the matrix $S_{\alpha, \epsilon}$ has 5000 (non-distinct) eigenvalues. Note the $x$-axis tallies the number of eigenvalues greater than 0.01, from which we observe that as $\epsilon$ increases, more and more eigenvalues are concentrated near zero. }
    \label{fig:TorusdiffmapEVs}
\end{figure}

\begin{figure}[h!]
    \centering
    \includegraphics[width=0.95\linewidth]{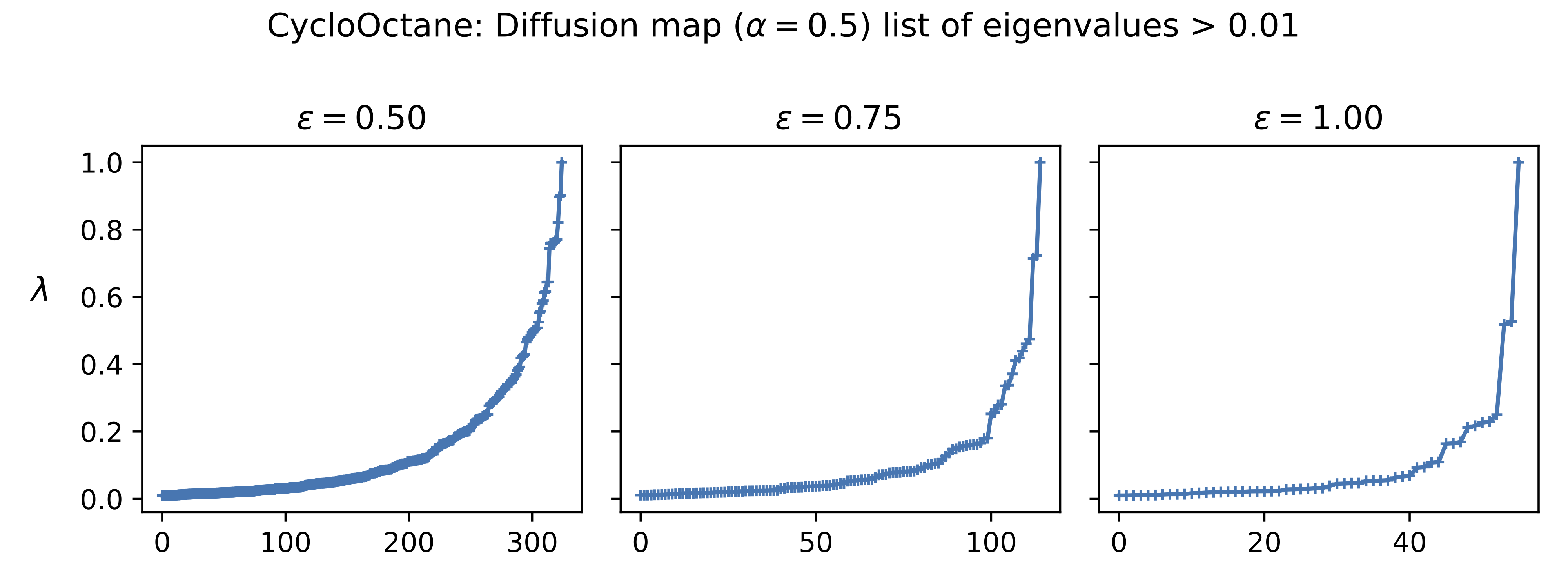}
    \caption{\emph{Cyclo-octane}. Eigenvalues of the diffusion map matrix $S_{\alpha, \epsilon}$, for $\alpha = 1/2$ and $\epsilon = 0.5, 0.75, 1.0$. We sort the eigenvalues between $[0,1]$ and plot those with $\lambda > 0.01$. 
As our input point cloud has 6040 points, the matrix $s_{\alpha, \epsilon}$ has 6040 (non-distinct) eigenvalues. Note the $x$-axis tallies the number of eigenvalues greater than 0.01, from which we observe that as $\epsilon$ increases, more and more eigenvalues are concentrated near zero. }
    \label{fig:COdiffmapEVs}
\end{figure}

\begin{figure}[h!]
    \centering
    \includegraphics[width=0.95\linewidth]{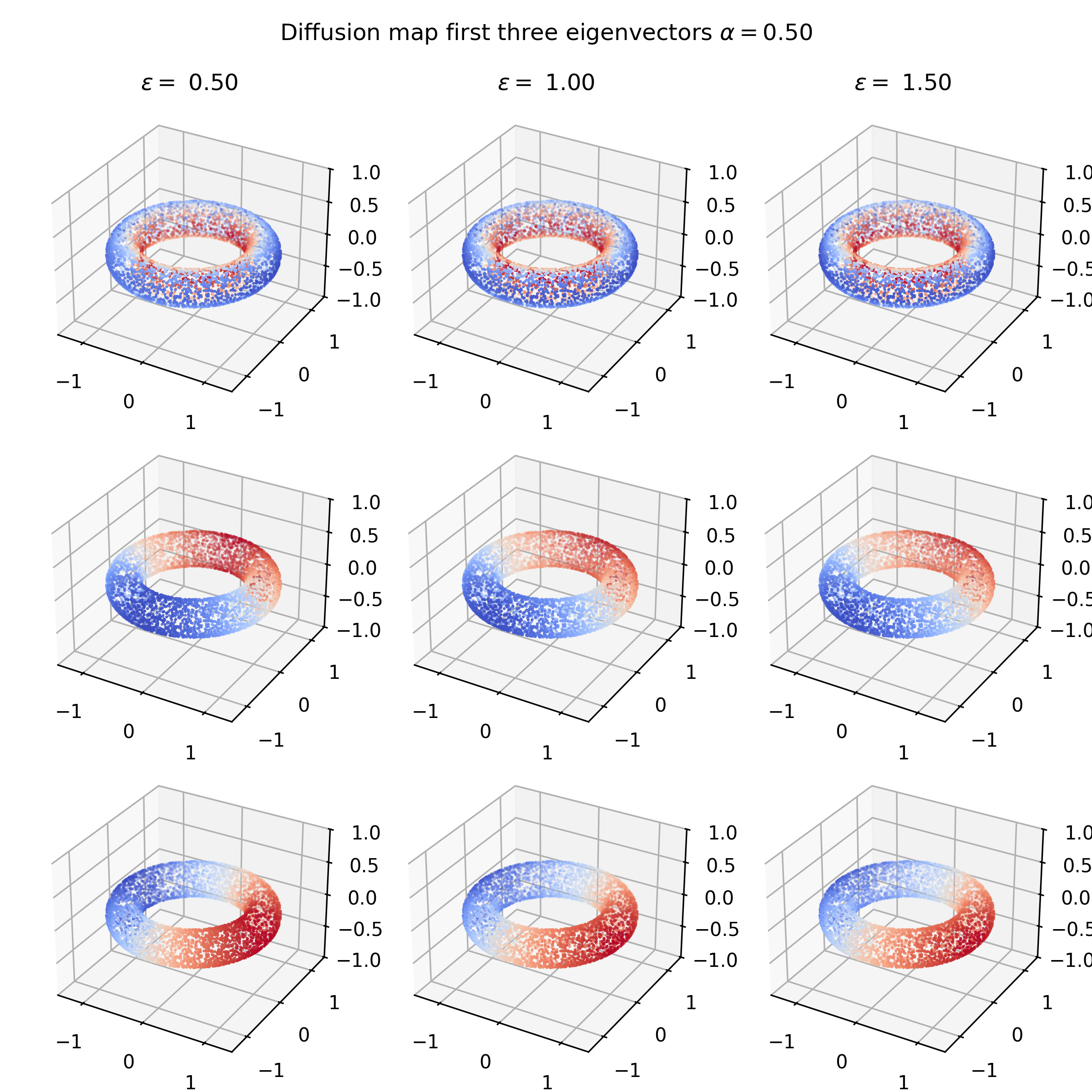}
    \caption{\emph{Torus}. We color the points on the torus by the value of their entries in the eigenvectors of the diffusion map matrix $s_{\alpha, \epsilon}$. We fix $\alpha = 1/2$ and vary $\epsilon = 0.5, 1.0, 1.5$. The values are color-scaled from blue to red. The first row represents the leading eigenvector with $\lambda = 1$, and the second and third rows represent the eigenvectors with the second and third largest eigenvalues. }
    \label{fig:TorusdiffmapEVecs}
\end{figure}

\begin{figure}[h!]
    \centering
    \includegraphics[width=0.95\linewidth]{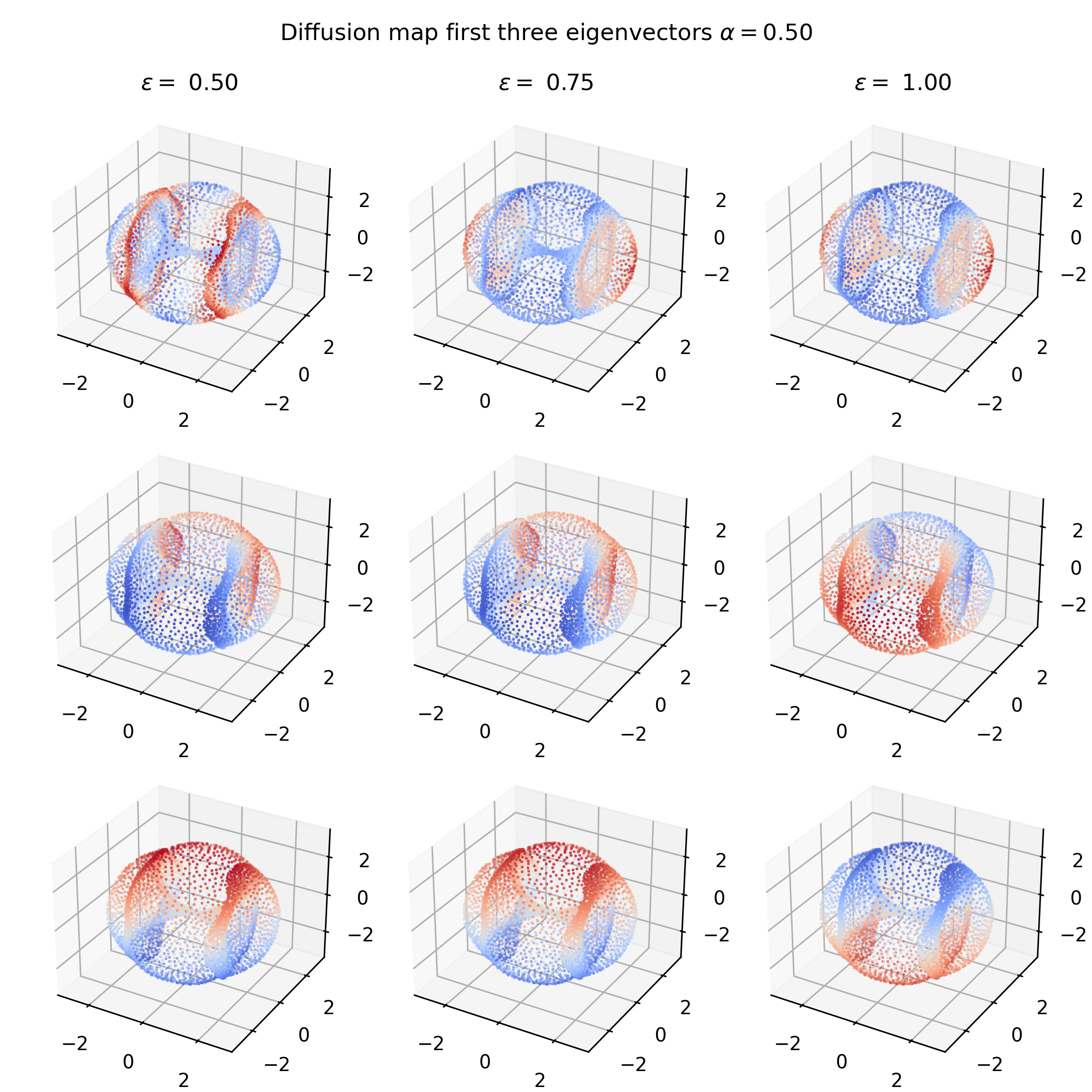}
    \caption{\emph{Cyclo-octane}. We color the points on the cyclo-octane point cloud (here visualized in three dimensions using Isomap) by the value of their entries in the eigenvectors of the diffusion map matrix $s_{\alpha, \epsilon}$. We fix $\alpha = 1/2$ and vary $\epsilon = 0.5, 0.75, 1.0$. The values are color-scaled from blue to red. The first row represents the leading eigenvector with $\lambda = 1$, and the second and third rows represent the eigenvectors with the second and third largest eigenvalues. }
    \label{fig:COdiffmapEVecs}
\end{figure}

\newpage
\section{GNN Experimental Parameters}\label{app:GNN}
All experiments were implemented in PyTorch Geometric~\cite{Fey2019FastGeometric}. The training procedure and the hyperparameters of the graph neural network are based on those in~\cite {Dwivedi2023BenchmarkingNetworks}.

\subsection{Training}
We use the Adam optimizer (with the default settings in PyTorch) to minimize the mean absolute error in the training. The initial learning rate is fixed to be  \texttt{1e-3}. We reduce the learning rate on plateaus, with parameters patience =  \texttt{10} epochs, learning rate reduction factor =  \texttt{0.5}, and we terminate the training when the learning rate drops below \texttt{1e-5}. 

We vary the batch size to be either 32 or 128 in our training. The reported results, which minimize the validation set error, are those with batch sizes as follows. The batch sizes for the trained models reported are all 128 for ZINC, and 32 for AQSOL, apart from the GAT architecture.

\subsection{GNN Architecture}
In all GNN models, several graph convolutional layers are applied before a pooling layer collects the vertex level features into a graph level feature. Then, graph level features are fed into a multilayer perceptron to yield the predicted feature. All hidden convolutional layers have constant dimensions, as do the dense layers in the final multilayer perceptron. The pooling layers either take the mean of the vertex features or the sum. Each hidden layer of the MLP consists of a batchnorm layer, a linear layer, followed by a ReLU activation. 

Since the vertex and edge features in both AQSOL and ZINC are categorical attributes, we use a one-hot encoding to represent each categorical feature as a vector in $\R^n$. The dimension of these feature spaces are denoted as vertex dim and edge dim in the tables below that describe the GNN architectures. 

All convolution layers retain residual connections. To avoid any confusion, the number of layers always refers to hidden layers, and corresponds to either the number of convolutions or non-linearities applied, depending on the context. 

\subsubsection{GAT}
For the GAT graph neural networks, the GAT convolutional layer parameters for the experiments on AQSOL and ZINC are shown in \Cref{tab:GAT_params}. Each GAT layer consists of a batchnorm layer, a GAT convolution layer, followed by the ELU non-linear activation. 

\begin{table}[h]

\centering
\begin{tabular}{|c|c|c|c|c|c|c|c|c|}
\hline
 GAT     & layers & dim & attn. heads & vertex dim &  edge dim & pool & MLP layers &  MLP dim \\ \hline
AQSOL & 4      & 18        & 8               & 144(119)              & 5             & mean & 1 & 144     \\ \hline
ZINC  & 4      & 18        & 8               & 144(114)          & 4             & mean & 1 & 144     \\ \hline
\end{tabular}
\vspace{5pt}
\caption{GAT architecture parameters. The number in brackets indicates the parameters for the case when features from the power spectra are incorporated. If brackets are not shown then the same parameters are used in both the experiments with and without power spectra features.}
\label{tab:GAT_params}
\end{table}

\subsubsection{GIN}
For the GIN graph neural networks, the GIN convolutional layer parameters for the experiments on AQSOL and ZINC are shown in \Cref{tab:GIN_params}. Each GIN layer consists of a batchnorm layer and a GIN convolution layer. Within each GIN layer, there is an MLP that transforms the features. 

In experiments in which power spectra are included as features, a separate MLP is first used to transform each quantile vector before feeding them into the convolution layers. This is because the native implementation of a GIN convolution layer in PyTorch Geometric adds incident edge features directly to vertex features at the initial pre-processing step. This dilutes the effect of the additional power spectra signature features. 

For AQSOL, the MLP has one hidden layer with dimension 45; for ZINC, the MLP has two hidden layers with dimension 30. 

When we trained the model, we varied where the $\epsilon$ parameter in the convolution layers is trainable or fixed to zero. We choose whether it is fixed or trainable based on the error on the validation set. For AQSOL, the $\epsilon$ parameter was trained for the experiment without power spectrum signatures, while the $\epsilon$ parameter fixed to zero yielded a better result when power spectrum signatures were included. For ZINC, the $\epsilon$ parameters were trained in both cases. 

\begin{table}[h]

\centering
\begin{tabular}{|c|c|c|c|c|c|c|c|c|}
\hline
GIN   & layers & dim & GIN MLP & GIN MLP & vertex dim & edge dim & pool & MLP dim \\ 
& & & layers & layer &  & & & \\ \hline
AQSOL & 4      & 110 & 1              & 110         & 110(65)    & 110      & sum  & 0       \\ \hline
ZINC  & 4      & 110  & 1              & 110         & 110(80)    & 110      & sum  & 0       \\ \hline
\end{tabular}
\vspace{5pt}
\caption{GIN architecture parameters. The number in brackets indicates the parameters for the case when features from the power spectra are incorporated. If brackets are not shown then the same parameters are used in both the experiments with and without power spectra features.}
\label{tab:GIN_params}
\end{table}

\subsubsection{GatedGCN}

For the GatedGCN graph neural networks, the GatedGCN convolutional layer parameters for the experiments on AQSOL and ZINC are shown in \Cref{tab:Gated_params}. Each GAT layer consists of a batchnorm layer, and a GatedGCN layer. 

\begin{table}[h!]

\centering
\begin{tabular}{|c|c|c|c|c|c|c|c|}
\hline
GatedGCN & layers & dim & vertex dim & edge dim & pool & MLP layer & MLP dim \\ \hline
AQSOL    & 4      & 90  & 90(65)     & 5        & mean & 1         & 90      \\ \hline
ZINC     & 4      & 70  & 70(40)     & 4        & mean & 0         & 0       \\ \hline
\end{tabular}
\vspace{5pt}
\caption{GIN architecture parameters. The number in brackets indicates the parameters for the case when features from the power spectra are incorporated. If brackets are not shown then the same parameters are used in both the experiments with and without power spectra features.}
\label{tab:Gated_params}
\end{table}

\vfill

\end{document}